\documentclass{article}  

\usepackage[a4paper]{geometry}
\usepackage[square, sort&compress, comma, numbers]{natbib}

\usepackage{graphicx}
\usepackage{subfigure}
\usepackage[outercaption]{sidecap}
\DeclareGraphicsExtensions{.eps} 
\graphicspath{{Figures/}{./}}

\usepackage[latin1]{inputenc}
\usepackage{fancyhdr}    
\usepackage{textcomp}  
\usepackage{amsmath} 
\usepackage{amssymb}
\usepackage{dsfont}
\usepackage{psfrag}
\usepackage{stfloats}
\usepackage{color}
\usepackage{multicol}
\usepackage{multirow}
\usepackage{algorithmic}
\usepackage{algorithm}
\usepackage{array}
\usepackage{bm}
\usepackage{paralist}{}
\usepackage{wasysym}

\usepackage{url}
\usepackage{stmaryrd}

\usepackage{bm}
\usepackage{flushend}
\usepackage{xspace}
\usepackage{color}
\usepackage{mathtools}

\allowdisplaybreaks

\usepackage{amsthm}

\newtheorem{proposition}{Proposition}

\newtheorem{definition}{Definition}
\newtheorem{lemma}{Lemma}
\newtheorem{corollary}{Corollary}

\newcommand{\mynumrob}{\ensuremath{N}\xspace}
\newcommand{\mynumrobs}{\ensuremath{N^{(s)}}\xspace}
\newcommand{\mynumrobi}[1]{\ensuremath{N^{(#1)}}\xspace}
\newcommand{\mynumspecies}{\ensuremath{N_S}\xspace}
\newcommand{\myspeciesset}{\ensuremath{\mathcal{S}}\xspace}
\newcommand{\myspeciessubset}{\ensuremath{\mathcal{\tilde{S}}}\xspace}

\newcommand{\mycrn}{\ensuremath{\mathcal{N}}\xspace}
\newcommand{\mystate}{\ensuremath{A}\xspace}
\newcommand{\mystateset}{\ensuremath{\mathcal{A}}\xspace}
\newcommand{\mynumstate}{\ensuremath{N_A}\xspace}
\newcommand{\mycomplex}{\ensuremath{C}\xspace}
\newcommand{\mycomplexset}{\ensuremath{\mathcal{C}}\xspace}
\newcommand{\mynumcomplex}{\ensuremath{N_C}\xspace}
\newcommand{\myreaction}{\ensuremath{R}\xspace}
\newcommand{\myreactionset}{\ensuremath{\mathcal{R}}\xspace}
\newcommand{\mynumreaction}{\ensuremath{N_R}\xspace}

\newcommand{\mya}[2]{\ensuremath{a_{(#1)}^{\{#2\}}}\xspace}
\newcommand{\myae}[1]{\ensuremath{a^{\{#1\}}}\xspace}
 
\newcommand{\myx}[2]{\ensuremath{x_{(#1)}^{\{#2\}}}\xspace}
\newcommand{\myxe}[1]{\ensuremath{x^{\{#1\}}}\xspace}

\newcommand{\myxss}[1]{\ensuremath{\bar{\mathsf{x}}_{#1}}\xspace}

\newcommand{\myxs}{\ensuremath{\mathsf{x}\xspace}}
\newcommand{\myxess}[1]{\ensuremath{\bar{\mathsf{x}}^{\{#1\}}}\xspace}

\newcommand{\mypropx}[1]{\ensuremath{r_{#1}}\xspace}

\newcommand{\myarrow}[2]{\ensuremath{\xrightleftharpoons[#2]{\,#1\,}}\xspace} 
\newcommand{\mynumobstate}{\ensuremath{N_O}\xspace}

\newcommand{\mymat}{\ensuremath{\mathbf{M}}\xspace}

\newcommand{\mydb}{\ensuremath{\mathcal{D}}\xspace}

\begin{document}

\title{A Macroscopic Model for Differential Privacy \\ in Dynamic Robotic Networks}

\author{Amanda Prorok and Vijay Kumar}
\date{}

\maketitle

\begin{abstract}
The increasing availability of online and mobile information platforms is facilitating the development of peer-to-peer collaboration strategies in large-scale networks. These technologies are being leveraged by networked robotic systems to provide applications of automated transport, resource redistribution (collaborative consumption), and location services.
Yet, external observations of the system dynamics may expose sensitive information about the participants that compose these networks (robots, resources, and humans). In particular, we are concerned with settings where an adversary gains access to a snapshot of the dynamic state of the system. We propose a method that quantifies how easy it is for the adversary to identify the specific \emph{type} of any agent (which can be a robot, resource, or human) in the network, based on this observation. We draw from the theory of \textit{differential privacy} to propose a closed-form expression for the leakage of the system when the snapshot is taken at steady-state, as well as a numerical approach to compute the leakage when the snapshot is taken at any given time. The novelty of our approach is that our privacy model builds on a macroscopic description of the system's state, which allows us to take account of protected entities (network participants) that are \emph{interdependent}. Our results show how the leakage varies, as a function of the composition and dynamic behavior of the network; they also indicate design rules for increasing privacy levels.
\end{abstract}


\section{Introduction}\label{sec:intro}
To date, the issues of privacy and security remain poorly addressed within robotics at large. These issues are  important in networked robotic systems (robot teams), where individual robots coordinate and plan their actions by communicating explicitly (e.g., through radio communication) or implicitly (e.g., through observations) with their network neighbors (in peer-to-peer mesh networks), and well as with with human operators or base-stations (in asymmetric broadcast architectures)~\citep{Michael:2011jz,Berman:2009vu}. When network communication channels are secured and data is encrypted, analysis of the network traffic flow may still reveal sensitive information and lead to privacy breaches~\citep{Shi:2004bm,Zhang:2011ia}. 
For example, the work in~\citep{Zhang:2011ia} shows that by observing a brief snapshot of network traffic, the activities of the network users were accurately inferred.
Privacy breaches may also incur if an adversary acquires aggregated data sets (e.g., by gaining access to the base-station). In particular, it has been shown that even if data is anonymized, adversaries can use independent anonymized data (i.e., side information) to breach privacy in so-called composition attacks~\citep{Ganta:2008df}.
Other work considers data that describes physically observable phenomena. For example, the work in~\citep{Ma:2013ig} investigates the security of mobility platforms that release anonymous mobility traces. The authors show that an adversary can infer the true identity of a user in a set of anonymous traces, by making use of only a small amount of side information. The work in~\citep{Chan:2008eh} considers a privacy preserving approach to the problem of monitoring crowds and estimating crowd sizes without making use of people models or tracking methods. Nevertheless, there is little work that describes how the privacy of physical systems is affected by the behaviors of the dynamic agents that compose them.

The overarching goal of our work is to ensure the anonymity of agents that compose a heterogeneous collaboration network. Working towards this goal, we present a method that allows us to quantify the loss of privacy that incurs when third parties can take a snapshot of the state of the dynamical system. The novelty of our work is that we quantify the loss of privacy of a dynamical system that is composed of \emph{connected and interdependent} robotic agents. This is not unlike the traditional setting of privacy, which considers computer databases of independent entries, with the difference that here, the individual entries (our robotic agents) are dependent.

Our framework makes use of a mean-field approach that models ensemble averages. The resulting macroscopic equation describes the system-wide dynamics, and allows us to compute the probability of any possible discrete system state. Subsequently, we use this measure to formulate a privacy metric that quantifies the maximum deviation of any two observable system state distributions.

\begin{definition}[\underline{Private Robot Network}]\label{def:private_system}
	\textit{ A private robot network is a team of heterogeneous robots where it is not possible to discern the type of an individual robot nor its specific interactions with other robots or human users.}
\end{definition}

\subsection{Related Work}
Various measures of privacy have been proposed in the database literature so far.
The early work in~\citep{Agrawal:2000gu} proposes a quantification of privacy in terms of the amount of noise added to a true value, or, in other words, how closely the original value of a modified attribute can be estimated. This measure, however, omits the notion of side-information, i.e., any additional information about the underlying distribution that the adversary might own. The work in~\citep{Evfimievski:2003jp} extends the notion of privacy to include such prior knowledge. The proposed measure suggests a quantification of the largest difference between an adversary's a-priori to a-posteriori beliefs (which corresponds to the worst-case scenario). It turns out that this model is significantly stronger, since it accounts for infrequent, but noticeable privacy breaches. 

In 2006, Dwork et al. introduced the notion of $\epsilon$-indistinguishability, a generalization of the measure in~\citep{Evfimievski:2003jp}, and later coined the term of \emph{differential privacy}~\citep{Dwork:2011tn}. Today, differentially private mechanisms are enjoying tremendous success, due to their ability of dealing with arbitrary side information (a \emph{future-proof} quality) and worst-case scenarios~\citep{Dwork:2008hs}. As a result, differentially private mechanisms are being developed and applied to various domains, including information networks~\citep{Koufogiannis:2016dc}, distributed convex optimization~\citep{Hale:2015ia}, Kalman filtering~\citep{LeNy:2014ba}, consensus algorithms~\citep{Katewa:2015jl}, smart grids~\citep{Koufogiannis:2014gb}, and traffic flow estimation~\citep{LeNy:2014kw}. Although these approaches present algorithms (mechanisms) that guarantee privacy levels, the common underlying assumption is that protected individuals or signals are all \emph{independent}. 

In particular, our work distinguishes itself from former approaches in that we present a framework that incorporates the dynamics of an interdependent networked system, and casts the probability distribution over all possible system states into a differential privacy formalism. In other words, we focus on developing a privacy metric for such dynamical systems; we do not, however, design algorithms to control or optimize their dynamics such that privacy constraints are met.

\subsection{Contributions}
The current paper presents a technique that allows us to analyze the privacy of networked systems composed of interdependent agents by modeling the loss of privacy when an external observer is able to gather information on the dynamic state of the system. We demonstrate the utility of our technique on hand of several case-studies. Specifically, we make the following contributions:

{\textit{1) \underline{Model of Networked Robotic System}:}} We begin by formulating a framework that allows us to capture dynamic collaborations of interdependent robots and resources.
We show how our model facilitates the design of collaborations and dependencies through \textit{compound states}, i.e., states that depend on multiple robots, humans or resources. This model also facilitates the definition of \textit{observable system-level} information, i.e., information that can be observed publicly. 

{\textit{2) \underline{Privacy Model}:}} The definition of privacy (or anonymity) is a difficult task, and a significant amount of research in the database literature is dedicated to this subject. A recent successful definition is that of \textit{differential privacy}~\citep{Dwork:2011tn}, which provides strong anonymity guarantees in the presence of arbitrary side information. One of our main contributions in this work is the development of an equivalent notion of privacy that can be applied to dynamic networks with \emph{interdependent} robots, humans and resources, where the goal is to protect the \emph{type} of individual agents in the network. Our measure quantifies the loss of privacy that incurs when system-wide observations are made. 

{\textit{3) \underline{Methods of Analysis}:}} Finally, we present a technique that employs the robotic network model and privacy model jointly to produce a quantitative analysis of privacy. The method uses a macroscopic description of the system dynamics, which is plugged into a formula of differential privacy. We show that in specific cases, the formula is closed-form, and can be computed efficiently. For the general case, we show how computational tools can be applied to evaluate the formula.

\section{Model of Networked Robotic System}

We define a networked system composed of robotic agents, where each agent belongs to a \textit{type}. The system is composed of \mynumspecies types $\myspeciesset= \{1,\ldots,\mynumspecies\}$, with a total number of \mynumrob robots, and \mynumrobs robots per type $s$ such that $\sum_{s \in \myspeciesset} \mynumrobs = \mynumrob$.
Robots interact (connect, collaborate, or create coalitions) a specific rate, which we assume to be known or to be determinable through system identification methods. 
A robot of type $s$ that is elementary (non-interactive) occupies a state denoted by $a^{\{s\}}_{(.)}$.
Robots that jointly interact share a state denoted $a_{(\cdot)}^{\mathcal{I}}$. The superscript $\mathcal{I}$ is the set of all robot types that are involved in this state, and the optional subscript denotes the specific state activity, if available (e.g., \emph{transporting}, \emph{sensing}, etc.). For example, $a_{(\mathrm{transport})}^{{\{1,2\}}}$ is a state where robots of type 1 and 2 collaborate to jointly transport goods.
Note that $\mathcal{I}$ may also be an empty set, which indicates that the state is unrelated to any particular robot type (such states could be byproducts that relate to performance metrics or environmental conditions, or could also be shared resources --- e.g., a battery pack, road lanes, or even network bandwidth).

The performance of networked robotic systems depends on collaborative mechanisms that require either physical or virtual interactions. 
Our aim is to capture these interactions, and hence, we choose a modeling framework that explicitly accounts for them. We build our formalism on the theory of Chemical Reaction Networks (CRN)~\citep{Feinberg:1991jd}, as it presents an efficient way of defining collaboration mechanisms with dependencies; simultaneously, CRNs provide tools to capture system-wide dynamics, enabling tractable analyses. Indeed, they are a powerful means of representing complex systems --- though not a new field of research, recent findings that quicken the computations are accelerating the adoption of CRNs into domains other than biology and chemistry~\citep{Munsky:2006es}. 

\subsection{Chemical Reaction Network}
We define our CRN as a triplet $\mycrn = (\mystateset, \mycomplexset, \myreactionset)$, where \mystateset is the set of states, \mycomplexset is the set of complexes, and \myreactionset is the set of reactions. 

{\underline{\textit{State set $\mathcal{A}$:}}} The state set encompasses all states that arise in the system, with $\mystateset = \{\mystate_1, \ldots, \mystate_{\mynumstate}\}$ where \mynumstate is the number of states. States relating to a specific robot type $s$ are denoted by $\mystateset^{(s)}$. The set of all states is denoted
\begin{equation}
		\mystateset = \mathop{\cup}_{s=1}^{\mynumspecies} \mystateset^{(s)}~~\mathrm{and}~~\mystateset^{(s)} =  \mathop{\cup}_{s \in \mathcal{I}} a^{\mathcal{I}}
\end{equation}
We can identify the compound states of an arbitrary subset of robots $\myspeciessubset \subset \myspeciesset$ by considering the intersection of sets $\mathop{\cap}_{i \in \mathcal \myspeciessubset} \mystateset^{(i)}$. Trivially, if $\mathop{\cap}_{i \in \mathcal \myspeciessubset} \mystateset^{(i)} = \emptyset$, then the robots in $\myspeciessubset$ do not collaborate.
The CRN is a population model, and allows us to keep track of the number of robots of any type in each of the states in \mystateset. Hence, we define a population vector $\mathbf{x} = [x_1, \ldots, x_{\mynumstate}] \in \mathbb{N}_{\geq 0}^{\mynumstate}$, where $x_i$ corresponds to the population present in state $\mystate_i$. We refer to the population vector $\mathbf{x}$ as the system-level state. In order to simplify the formulation of our case studies later on, we will also use the notation ${x}^{\mathcal{I}}$ to refer explicitly to the population in state $a_{}^{\mathcal{I}}$.

\underline{\textit{Complex set $\mathcal{C}$:}} 
The complex set is defined as $\mycomplexset=\{\mycomplex_1,\ldots, \mycomplex_{\mynumcomplex}\}$, with \mynumcomplex the number of complexes, and where $\mycomplex_j = \sum_{i=1}^{\mynumstate} \rho_{ij} \mystate_i$ for $j = 1, \ldots, \mynumcomplex$, with vector $\boldsymbol{\rho_{j}} = [\rho_{1j}, \ldots, \rho_{\mynumstate j}]^{\top} \in \mathbb{N}_{\geq 0}^{\mynumstate}$.
A complex is a linear combination of states, and denotes the net input or output of a reaction. 
In other words, a complex denotes either \emph{(i)} the states that are required for a certain reaction to take place, or \emph{(ii)} the states that occur as an outcome of a certain reaction that took place. The non-negative integer terms $\rho_{ij}$ are coefficients that represent the multiplicity of the states in the complexes.

\underline{\textit{Reaction set $\mathcal{R}$:}}
We use complexes to formulate reactions
$\myreaction_l : \mycomplex_j \xrightarrow{\mypropx{l}} \mycomplex_k.$
The reaction set is defined as $\myreactionset = \{\myreaction_1, \ldots, \myreaction_{\mynumreaction}\}$, with \mynumreaction the number of reactions, such that $\myreaction_l \in \{(\mycomplex_j, \mycomplex_k) | \exists \, \mycomplex_j, \mycomplex_k$ with $\mycomplex_j \rightarrow \mycomplex_k\}$ for $j, k = 1,\ldots,\mynumcomplex$, and where $r_l$ is the rate function $r_l(\mathbf{x}; \kappa_l): \mathbb{N}_{\geq 0}^{\mynumstate} \mapsto \mathbb{R}_{\geq 0}$ parameterized by rate constant $\kappa_l$.
In this work, we use non-linear mass-action rate functions, and $r_l(\mathbf{x};\kappa_l) =  \kappa_l \prod_{i=1}^{\mynumstate} x_i^{\rho_{ij}}$ for all $\myreaction_l = (\mycomplex_j,\cdot)$.
A set of complexes that is connected by reactions is termed a linkage class.
%
The net loss and gain of each reaction is summarized in a $\mynumstate \times \mynumreaction$ stoichiometry matrix $\Gamma$, the columns of which encode the change of population per reaction.
In particular, the $i$-th column of $\Gamma$ corresponds to the $i$-th reaction $\myreaction_i = (\mycomplex_j, \mycomplex_k)$ and thus, the column is equal to $\boldsymbol{\rho_{k}} - \boldsymbol{\rho_{j}}$.
The elements $\Gamma_{ji}$ are the so-called stoichiometric coefficients of the $j$-th state in the $i$-th reaction. Positive and negative coefficients denote products and reactants of the reaction, respectively.

\subsection{System Dynamics}\label{sec:dynamics}

We take a mean-field approach to model the system deterministically, and represent robot ensemble averages with \myxs{} (unlike $\mathbf{x}$ that denotes the discrete state).
The average population value in the respective system states changes according to an ordinary differential equation, described as follows
\begin{eqnarray}\label{eq:dynamics}
	\dot{\myxs} = \mymat \mathbf{A}\psi(\myxs),
\end{eqnarray}
where $\psi(\myxs)$ returns a vector in $\mathbb{R}^{\mynumcomplex}$ in which each entry $\psi_j$ is the product of states in complex $j$ (i.e., $\psi_j = \prod_{i=1}^{\mynumstate} x_i^{\rho_{ij}}$), where $\mymat \in \mathbb{R}^{\mynumstate \times \mynumcomplex}$ is a matrix in which each entry $\mymat_{ij}$ is the coefficient of state $j$ in complex $i$, and where matrix $\mathbf{A} \in \mathbb{R}^{\mynumcomplex \times \mynumcomplex}$ is defined as
\[ 
\mathbf{A}_{ij} = \left \{
  \begin{tabular}{ll}
  $\kappa_{ji}$, &{if $i \neq j, (C_i,C_j) \in \myreactionset$}\\
  0, &{if $i \neq j, (C_i,C_j) \notin \myreactionset$} \\
  $\mathop{-\sum}\limits_{(\mycomplex_i, \mycomplex_k)\in\myreactionset} \kappa_{ki}$, &if $i=j$
  \end{tabular}
\right.
\]
This general form captures non-linear dynamics. We note that when agents do not interact (and there are no dependencies), the  system exhibits linear dynamics. In this case, each individual state is also a complex, with $\mynumstate=\mynumcomplex$. The matrix $\mathbf{M}$ is the identity matrix, and the function $\psi(\myxs)=\myxs$. 

\subsection{Continuous-Time Markov Chain}\label{sec:markov_chain}
We can also model the system stochastically, and keep track of the exact number of robots in system states. A simple stochastic model for CRNs treats the system as a continuous time Markov chain with state $\mathbf{x} \in \mathbb{N}_{\geq 0}^{\mynumstate}$ (i.e., the population vector), and with each reaction modeled as a possible transition for the state. 
This model assumes that the time between transitions is exponentially distributed, and hence, the number of transitions between two neighboring states is Poisson distributed. In order to calibrate rate constants $\kappa_l$ on hand of a real system, we can proceed by measuring the effective transition rates (by observing the number of transitions, assuming the number of robots is known), and using the mass-action rate functions to solve for the parameter values.
The Chemical Master Equation (CME)~\citep{LopezCaamal:2014kj} describes the temporal evolution of the probability mass function over all possible population vectors, and is given by a set of ordinary differential equations associated to the continuous-time, discrete-state Markov Chain. 
The CME is given by the linear ordinary differential equation
\begin{eqnarray}\label{eq:cme}
\dot{\boldsymbol{\pi}}(t) = K \boldsymbol{\pi}(t)
\end{eqnarray}
with $\boldsymbol{\pi} = [\pi_{\mathbf{x}_i} | \mathbf{x}_i \in \mathcal{X}_{\mycrn}] $ and where $\mathcal{X}_{\mycrn}$ is the set of all possible population vectors $\mathbf{x}$ that can arise from the CRN \mycrn.
The entries of the transition rate matrix $\mathbf{K} \in \mathbb{R}^{|\mathcal{X}_{\mycrn}| \times |\mathcal{X}_{\mycrn}|}$ are given by
\begin{equation}\label{eq:transition_rates}
\mathbf{K}_{ij} = \left \{
  \begin{tabular}{ll}
  $-\sum_{m=1}^{\mynumreaction} r_m(\mathbf{x}_i;\kappa_m)$, &{if $i =j$}\\
  $r_m(\mathbf{x}_i;\kappa_m)$, &{$\forall j: \mathbf{x}_j = \mathbf{x}_i + \boldsymbol{\rho}_l - \boldsymbol{\rho}_k$} \\
  & $\mathrm{with}\, R_m = (\mycomplex_k, \mycomplex_l)$ \\
  $0$, &otherwise
  \end{tabular}
\right.
\end{equation}
When the number of possible system-wide states $|\mathcal{X}_{\mycrn}|$ is small, it is possible to obtain a closed-form solution to Eq.~\eqref{eq:cme}. However, when $|\mathcal{X}_{\mycrn}|$ is large or even infinite, it may become computationally intractable to solve the system. In such cases, we can resort to Finite State Projection (FSP) methods~\citep{Munsky:2006es} that approximate the solution by compressing the number of possible states (and, hence, also the size of $\mathbf{K}$). 
The idea of FSP is to expand the number of states dynamically, according their probabilities. States with low probabilities are pruned, and, hence, only statistically relevant states are added to the domain of the solver. 

\subsection{Observable System-Level State}
Finally, we introduce two auxiliary functions that help us describe the system behavior: a function $f_{\mycrn}$ that describes the system dynamics, and a query function $q$:
\begin{eqnarray}
f_{\mycrn}(\mathbf{x_0}, t): && \mathbb{N}_{\geq 0}^{\mynumstate} \times \mathbb{R}_{\geq 0}  \mapsto  \mathbb{N}_{\geq 0}^{\mynumstate} \nonumber \\
q(\mathbf{x}): && \mathbb{N}_{\geq 0}^{\mynumstate}  \mapsto  \mathbb{N}^{\mynumobstate},\, \mynumobstate \in \mathbb{N}_{> 0} 
\end{eqnarray}
The execution function $f_{\mycrn}$ samples a trajectory (up to time $t$) of system-level states, as given by the continuous-time Markov chain, and returns a population vector $\mathbf{x}(t)$.

The query function $q$ allows us to formalize the notion of an \textit{observable system-level} state. It takes the population vector $\mathbf{x}$ as input, and returns a vector of \emph{observable} values $\mathbf{y}$. 
In its most basic form, the query function is the identity function, meaning that an observer is able to capture the exact (true) system-level state, and $\mathbf{x} = \mathbf{y}$. In this work, we assume that the observed values take the form of simple summations over the population vector. In particular, robots in their elementary states $a_{(\cdot)}^{\{s\}}$ are not identifiable, since this would lead to a direct breach of privacy.
This assumption is well motivated when the types of individual robots are not distinguishable from an outside vantage point, and thus, only aggregated values can be observed. The components of $\mathbf{y}$ are given by
\begin{equation}\label{eq:obs_state}
y_i = \sum_{j \in \Omega_i} x_j
\end{equation}
with $\Omega_i \subset \{1,\ldots,\mynumstate\}$ and all $\Omega_i$ disjoint.

\section{A Motivational Example}
\begin{SCfigure}[][b]
\psfrag{A}[cc][][0.8]{$a^{\{A\}}$}
\psfrag{B}[cc][][0.8]{$a^{\{B\}}$}
\psfrag{R}[cc][][0.8]{$a^{\{R\}}$}
\psfrag{C}[cc][][0.8]{$a^{\{A,R\}}$}
\psfrag{D}[cc][][0.8]{$a^{\{B,R\}}$}
\psfrag{1}[cc][][0.8]{$r_1, r_2$}
\psfrag{2}[cc][][0.8]{$r_3, r_4$}
{\includegraphics[width=0.3\columnwidth]{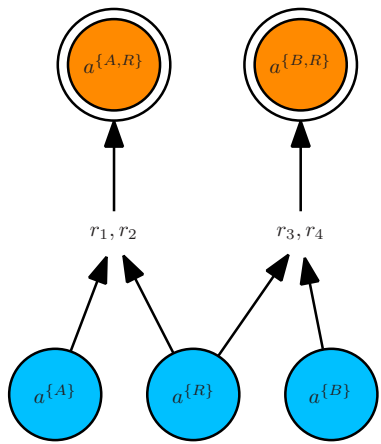}}
\caption{Reaction topology for the example in Eq.~\eqref{eq:example_crn}. 
Doubly outlined circles represent compound states.
\label{fig:example_graph}} 
\end{SCfigure}
Let us consider a simple example. Our robot network consists of two robot types $\{A, B\}$ that access a resource $R$, which is available to all robots in the network. The goal of this example is to illustrate that we can isolate the type of a particular robot, just by observing the dynamic state of the system.

The system is defined as follows. The robots of type $A$ and $B$ utilize the resource $R$ at rates \mypropx{1} and \mypropx{3}, respectively, and they return the resource at rates \mypropx{2} and \mypropx{4}, respectively. The CRN of this system is formulated as
\begin{eqnarray}\label{eq:example_crn}
	a^{\{A\}} + a^{\{R\}} \myarrow{\mypropx{1}}{\mypropx{2}} a^{\{A,R\}} \nonumber\\
	a^{\{B\}} + a^{\{R\}}\myarrow{\mypropx{3}}{\mypropx{4}} a^{\{B,R\}},
\end{eqnarray}
where $a^{\{A\}}$ and $a^{\{B\}}$ denote the state of a robot of type $A$ or $B$ before the resource is allocated, and $a^{\{A,R\}}$ and $a^{\{B,R\}}$ denote the state of the robots after obtaining the resource. 
The internal system state representation is $\mathbf{x} = [x^{\{A\}},\, x^{\{B\}},\, x^{\{R\}},\, x^{\{A,R\}}, \,x^{\{B,R\}}]^\top$. Fig.~\ref{fig:example_graph} illustrates the topology of this CRN.
In our example, the system composed of three robots, with two alternative instantiations, defined by \emph{adjacent} databases \mydb and \mydb'. These two databases differ by a single row, i.e., the type of robot with ID 2. Figure~\ref{fig:example_DB} illustrates the system.
 We assume that two instances of the resource $R$ are available to the robots. Initially, no resources are allocated, hence, for database \mydb, we have $\mathbf{x_0} = [2,\, 1,\, 2,\, 0, \,0]^\top$, and for database \mydb', we have $\mathbf{x_0} = [1,\, 2,\, 2,\, 0, \,0]^\top$. 
We recall that system transition rates are defined as mass-action functions, with the transition rate matrix given by Eq.~\eqref{eq:transition_rates}. The continuous time Markov chain corresponding to the two databases is shown in Fig.~\ref{fig:example_CTMC}. 
\begin{SCfigure}
\psfrag{z}[lc][][0.7]{Robot ID}
\psfrag{e}[cc][][0.7]{Database \mydb}
\psfrag{f}[cc][][0.7]{Database \mydb'}
\psfrag{o}[cc][][0.7]{Type}
\psfrag{5}[cc][][0.7]{1}
\psfrag{6}[cc][][0.7]{2}
\psfrag{7}[cc][][0.7]{3}
\psfrag{k}[lc][][0.7]{$A$}
\psfrag{m}[lc][][0.7]{$B$}
{\includegraphics[width=0.6\columnwidth]{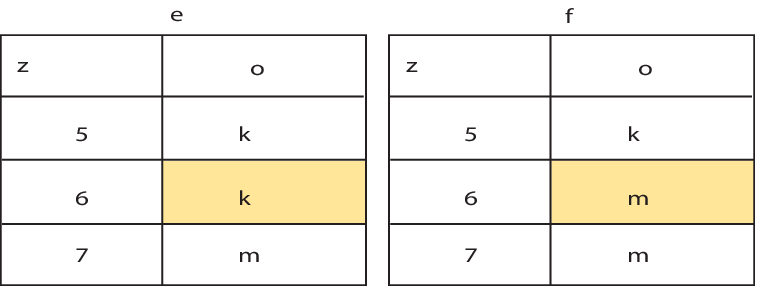}}
\caption{We consider two alternative systems, recorded in databases \mydb and \mydb'. Robot no. 2 once belongs to type $A$, and once to type $B$.
\label{fig:example_DB}} 
\end{SCfigure}
We define the observable system-level state as
\begin{equation}
	\mathbf{y} = \begin{bmatrix} x^{\{A\}} + x^{\{B\}} \\ x^{\{A,R\}} + x^{\{B,R\}} \end{bmatrix},
\end{equation}
which counts the number of robots using and not-using the shared resource.
We note that state $x^{\{R\}}$ is factored out, since we are only interested in protecting the types of the robots.

The observable distribution of the system at steady-state~\footnote{By Theorem 4.1 of~\citep{Siegel:2000ui}, we can show that this system has a unique, globally asymptotically stable positive equilibrium.} for both database variants \mydb and \mydb' is shown in Fig.~\ref{fig:example_PMF}. Although the composition differs only in one entry, the distributions differ significantly. Now, let us imagine that an adversary makes an observation of the system: 2 robots are currently not using the resource, and 1 robot currently is; hence, he concludes that the state is $\mathbf{x} = [2,\,1]^\top$. The adversary may own side-information, for example, he may know the underlying system dynamics, and he can compute the expected distribution of the system at steady-state. 
Building on this, he can infer that the more likely database is \mydb'. In the worst case --- i.e., the adversary knows the types of both robots \emph{not} using the resource --- he can then directly infer the type of the robot using the resource. For  example, if the two robots not using the resource are type $A$ and $B$, then the probability that the robot using the resource  belongs to type $A$ is 0.43, and the probability that it belongs to type $B$ is 0.57.

On hand of this basic example, we demonstrated that the observable state of a system is affected by its instantiation (database), and that it is easy to distinguish a system \mydb from \mydb' (and hence isolate any robot) if the likelihood that an observation belongs to either \mydb or \mydb' is very different. 
Thus, if we want to conceal the types of robots in the system, we need to ensure that the deviation in observable distributions is kept as small as possible for all adjacent databases. We address this problem by first developing a model that quantifies the \emph{loss of privacy}, as a function of the structure of the collaboration mechanism, the collaboration rates, and the system's composition.

\begin{figure}
\vspace{0.5cm}
\centering
\psfrag{1}[cc][][0.6]{$\begin{bmatrix} 2 \\ 1 \\ 2 \\ 0 \\ 0 \end{bmatrix}$}
\psfrag{2}[cc][][0.6]{$\begin{bmatrix} 1 \\ 1 \\ 1 \\ 1 \\ 0 \end{bmatrix}$}
\psfrag{3}[cc][][0.6]{$\begin{bmatrix} 0 \\ 1 \\ 0 \\ 2 \\ 0 \end{bmatrix}$}
\psfrag{4}[cc][][0.6]{$\begin{bmatrix} 2 \\ 0 \\ 1 \\ 0 \\ 1 \end{bmatrix}$}
\psfrag{5}[cc][][0.6]{$\begin{bmatrix} 1 \\ 0 \\ 0 \\ 1 \\ 1 \end{bmatrix}$}
\psfrag{0}[cc][][0.6]{$\begin{bmatrix} 1 \\ 2 \\ 2 \\ 0 \\ 0 \end{bmatrix}$}
\psfrag{6}[cc][][0.6]{$\begin{bmatrix} 1 \\ 1 \\ 1 \\ 0 \\ 1 \end{bmatrix}$}
\psfrag{7}[cc][][0.6]{$\begin{bmatrix} 1 \\ 0 \\ 0 \\ 0 \\ 2 \end{bmatrix}$}
\psfrag{8}[cc][][0.6]{$\begin{bmatrix} 0 \\ 2 \\ 1 \\ 1 \\ 0 \end{bmatrix}$}
\psfrag{9}[cc][][0.6]{$\begin{bmatrix} 0 \\ 1 \\ 0 \\ 1 \\ 1 \end{bmatrix}$}
\psfrag{a}[lc][][0.7]{$4\kappa_1$}
\psfrag{b}[lc][][0.7]{$\kappa_2$}
\psfrag{c}[lc][][0.7]{$\kappa_1$}
\psfrag{d}[lc][][0.7]{$2\kappa_2$}
\psfrag{o}[lc][][0.7]{$2\kappa_3$} \psfrag{e}[lc][][0.7]{$\kappa_4$}
\psfrag{n}[lc][][0.7]{$\kappa_3$} \psfrag{m}[lc][][0.7]{$\kappa_4$}
\psfrag{f}[lc][][0.7]{$2\kappa_1$} \psfrag{g}[lc][][0.7]{$\kappa_4$}
\psfrag{i}[lc][][0.7]{$4\kappa_3$}
\psfrag{x}[lc][][0.7]{$2\kappa_4$}
\subfigure[]{\includegraphics[width=0.4\columnwidth]{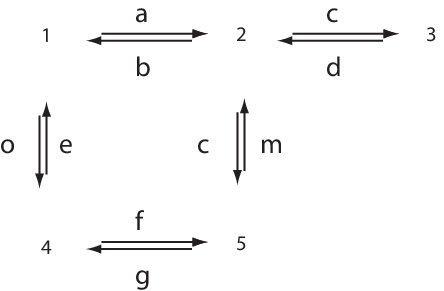}}\hspace{1.9cm}
\subfigure[]{\includegraphics[width=0.4\columnwidth]{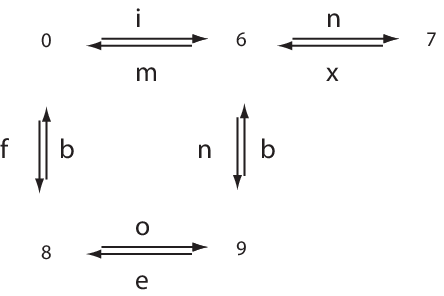}}
\caption{Continuous time Markov chains as defined by Eq.~\eqref{eq:example_crn}, for two variant systems, (a) \mydb and (b) \mydb'. 
\label{fig:example_CTMC}} 
\vspace{0.2cm}
\end{figure}

\begin{SCfigure}
\psfrag{a}[tc][][0.7]{$\mathbf{y}=\begin{bmatrix} 1 \\ 2 \end{bmatrix}$}
\psfrag{n}[tc][][0.7]{$\mathbf{y}=\begin{bmatrix} 2 \\ 1 \end{bmatrix}$}
\psfrag{c}[tc][][0.7]{$\mathbf{y}=\begin{bmatrix} 3 \\ 0 \end{bmatrix}$}
\psfrag{o}[lc][][0.65]{$\lim\limits_{\tau \rightarrow \infty} \mathbb{P}[q \circ f_{\mycrn}(\mathbf{x}_0(\mydb),\tau)]$}
\psfrag{e}[lc][][0.65]{$\lim\limits_{\tau \rightarrow \infty} \mathbb{P}[q \circ f_{\mycrn}(\mathbf{x}_0(\mydb'),\tau)]$}
\psfrag{p}[cc][][0.7][90]{$\mathbb{P}$}
\centering
{\includegraphics[width=0.55\columnwidth]{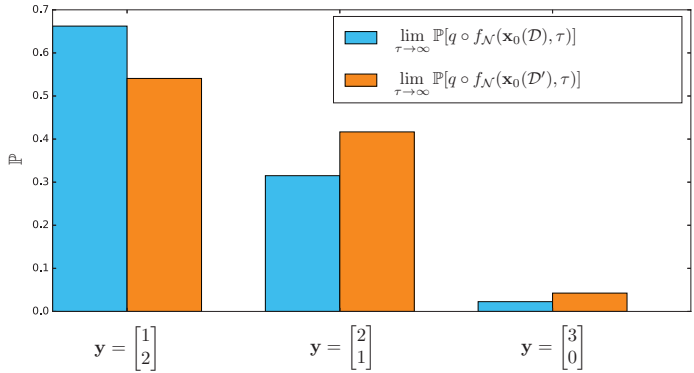}}
\caption{Observable steady-state distributions, for databases \mydb and \mydb' as defined in Fig.~\ref{fig:example_DB}. The rate constants are $\kappa_1 = 3$, and $\kappa_2 = \kappa_3 = \kappa_4 = 1$.  
\label{fig:example_PMF}} 
\vspace{0.5cm}
\end{SCfigure}

\section{Differentially Private Robot Networks}
In this section, we develop our analogy to a formal definition of privacy that stems from the database literature, and that is referred to as differential privacy (formerly known as indistinguishability)~\citep{Dwork:2011tn}.
This concept considers two key components: a \textit{database} that holds sensitive information pertaining to individuals, and a \textit{query} that releases information obtained from the database via a mechanism.
The goal of differential privacy is to develop mechanisms that are able to provide information in response to database queries, while preserving the privacy of the individuals recorded therein, even in the presence of arbitrary side information\footnote{In our context of a networked robotic system, an example of side information could be the number of manufacturing parts ordered to build the system. If different robot species are made of different parts, such information can be used to construct an initial guess about the number of robots per species. Thus, one would be able to derive the probability of a robot belonging to a given species.}. 
Side information can be understood as a prior probability distribution over the database, and hence, privacy is preserved if no additional information about this distribution is obtained through the query.
It is important to note that the condition of differential privacy is made with respect to the release mechanism (i.e., query), and does not depend on the database itself, nor on the side information. In particular, if an individual's presence or absence in the database does not alter the distribution of the output of the query by a significant amount, regardless of the side information, then the privacy of that individual's information is assured.

Our analogy applies the concepts of database and query to the context of networked robotic systems. We consider a database that represents the composition of our robotic system, and that records the types of each of the robots. Also, we consider an external observer who is capable of observing the robotic system during its operation, and who can query the system by retrieving information about the system-level state (i.e., observable system-level state). Then, our analogous definition of privacy is the notion that the observer cannot obtain private information about individual robots by querying the system (i.e., the type of a robot remains private, cf. Def.~\ref{def:private_system}).
The composition of our robotic network is recorded in a database $\mydb \in \myspeciesset^{\mynumrob}$ that consists of \mynumrob entries, where each entry $\mydb_i$ denotes the type of robot $i$.
We define an \textit{adjacency} set $\mathrm{Adj}(\mydb)$ that encompasses all databases \mydb' adjacent to \mydb. 
Two databases \mydb and \mydb' are adjacent if they differ by one single entry. In other words, two robotic networks (represented by \mydb and \mydb') are adjacent if they differ by one robot $i$, meaning that robot $i$ belongs to $s_i$ in \mydb (i.e., $\mydb_i = s_i$), and to a different type $s'_{i} \neq s_i$ in \mydb' (i.e., $\mydb_i \neq s_i$). 
As previously described, the behavior of the robotic network can be described by tracking the states that compose the CTMC. If we let the system run, it produces a trajectory that can be evaluated at a given time $\tau$, resulting in a snapshot of the population vector $\mathbf{x}$.
Our query/response model consists of an external user (adversary) who is able to observe the system-level state at time $\tau$ --- we refer to this model as a snapshot adversary. 
\begin{definition}[\underline{Snapshot Adversary}]\label{def:adversary}
A snapshot adversary gains system-level information at a specific time $\tau$. This system-level information corresponds to the \emph{observable} state $\mathbf{y} = q \circ f_{\mycrn}(\mathbf{x_0}(\mydb), \tau)$. A special case of snapshot adversary is an adversary who queries the system once, at any time during its nominal (steady-state) regime. This special case is referred to as the \underline{steady-state snapshot adversary}.
\end{definition}
The query $q \circ f_{\mycrn}(\mathbf{x_0}(\mydb), \tau)$ depends on the system's composition \mydb, and the time at which the system is observed $\tau$. 
The function $\mathbf{x_0}(\mydb): \myspeciesset^{\mynumrob} \mapsto \mathbb{N}_{\geq 0}^{\mynumstate}$ distributes the robots in \mydb to their initial states.
A schema of this system is shown in Fig.~\ref{fig:schema}. Our aim is to analyze the differential privacy of the observed system output. To this end, we propose a definition of differential privacy that is applicable to dynamic networked robotic systems.
\begin{figure}[tb]
\psfrag{d}[cc][][0.7]{$\mydb$}
\psfrag{a}[cc][][0.7]{$f_{\mycrn}(\mathbf{x_0}(\mydb),\tau)$}
\psfrag{q}[cc][][0.7]{$q \circ f_{\mycrn}(\mathbf{x_0}(\mydb),\tau)$}
\psfrag{t}[cc][][0.7]{$\tau$}
\psfrag{x}[cc][][0.7]{$\mathbf{x}$}
\psfrag{y}[cc][][0.7]{$\mathbf{y}$}
\centering
{\includegraphics[width=0.75\columnwidth]{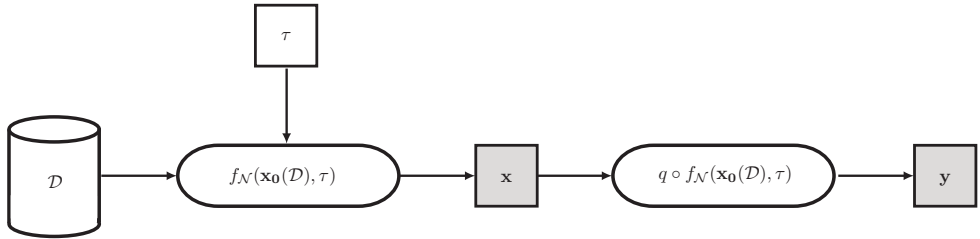}}
\caption{The composition of the robotic network is recorded in a database \mydb. The function $f_{\mycrn}$ is the stochastic process governing the dynamics of the robotic network, and returns a system-level state output $\mathbf{x}$. Query function $q$ reads the (internal) system-level state, and returns the observable output $\mathbf{y}$. Parameter $\tau$ denotes the time at which the system is observed.
\label{fig:schema}} 
\end{figure}

\begin{definition}[\underline{$\epsilon$-indistinguishable robot network}]
\textit{A networked robot team, which, according to database \mydb, is composed of heterogeneous robot types that interact with dynamics defined by a system \mycrn, is $\epsilon$-indistinguishable (and gives $\epsilon$-differential privacy) if for all possible team compositions recorded in databases \mydb', we have}
\begin{equation} \label{eq:diff_priv}
\mathcal{L}(\mydb) = \mathop{\mathrm{max}}\limits_{\mydb' \in \mathrm{Adj}(\mydb)} \left | \mathrm{ln} \frac{\mathbb{P}[q \circ f_{\mycrn}(\mathbf{x_0}(\mydb),\tau)]}{\mathbb{P}[q \circ f_{\mycrn}(\mathbf{x_0}(\mydb'),\tau)]} \right | \leq \epsilon.
\end{equation}
\textit{where $\mathbb{P}[\mathbf{y}]$ denotes the probability of the output $\mathbf{y}$, obtained through query $q$ of the system-level state given by $f_{\mycrn}$.}
\end{definition}
The value $\epsilon$ is referred to as the leakage. Intuitively, this definition states that if two robotic systems are similar, in order to preserve privacy they should correspond to similar distributions over their observable outputs. As noted by Dwork et al.~\citep{Dwork:2011tn}, the definition of differential privacy is stringent: for a pair of distributions whose statistical difference is arbitrarily small, the ratio may be large, even infinite in the case when a point in one distribution assigns probability zero and the other non-zero. 
Later, in our evaluations, we use a smooth version of the leakage formula above, by adding an arbitrary, negligibly small value $\nu$, uniformly over the support of the probability distributions. This allows us to differentiate between large and small mismatches of the output when one point of a probability distribution returns zero. Due to this addition, we are able to show continuous privacy trends as a function of the underlying parameters.

\section{Method}\label{sec:method}
The formula in Eq.~\eqref{eq:diff_priv} provides strong privacy guarantees. Yet, it requires that we have a way of specifying the probability distribution over the system's observable output. The choice of method for computing this probability distribution depends on several factors, such as system dynamics, time of observation (during transient vs. stationary states), and computational complexity. Here, we will consider two alternative methods, \emph{(i)} one that yields a closed-form equation, but is only applicable to a certain class of system that admits a steady-state snapshot adversary, and \emph{(ii)} one that relies on numeric approximations, but that is generally applicable, and that admits a generic snapshot adversary (see Def.~\ref{def:adversary}).

\subsection{Complex-Balanced Mechanisms}\label{sec:CB} 
We present a method that computes the probability distribution at steady-state, which corresponds to the dynamical system's operational mode.
We show how this can be done very efficiently for a class of CRNs whose stationary distribution can be formulated analytically: \emph{complex-balanced} CRNs.

\begin{lemma}[Th. 3.2~\citep{Anderson:2010ux}]\label{th:CB}
If the system in Eq.~\eqref{eq:dynamics} admits $\mathbf{A}\psi(\myxss{}) = \mathbf{0}$, then the system is complex-balanced, with equilibrium point $\myxss{} \in \mathbb{R}^{\mynumstate}$.
There is one, and only one, equilibrium point in each positive stoichiometric compatibility class; each equilibrium point of a positive stoichiometric compatibility class is locally asymptotically stable relative to its stoichiometric compatibility class.
\end{lemma}

For systems that satisfy Lemma~\ref{th:CB}, we can follow Theorem (4.2) of~\citep{Anderson:2010ux} to define the stationary distribution $\bar{\pi}_{\mydb}(\mathbf{x}) = \lim_{t \rightarrow \infty} \mathbb{P}[f_{\mycrn}(\mathbf{x_0}(\mydb),t)]$ of the stochastically modeled system.
If the system is irreducible, this stationary distribution consists of a product of Poisson distributions and is given by
\begin{eqnarray} \label{eq:crn_stationary}
	\bar{\pi}_{\mydb}(\mathbf{x}) = \prod_{i=1}^{\mynumstate} \frac{\myxss{i}^{x_i}}{x_i !} e^{-\myxss{i}},~\mathbf{x}\in \mathcal{X}_{\mycrn}(\mydb)
\end{eqnarray}
where $\mathcal{X}_{\mycrn}(\mydb)$ is the set of all possible population vectors $\mathbf{x}$ that can arise from the CRN \mycrn and the robot species specified by \mydb. We note that when the system is reducible, a similar equation exists, see~\citep{Anderson:2010ux}.

We use this description to formulate a closed-form measure of the loss of privacy.
\begin{proposition}\label{th:leakage}
{If a networked robotic system defined by database \mydb and CRN \mycrn is complex-balanced and irreducible, and is observed through the identity query $q_{\mycrn}(\mathbf{x}) = \mathbf{x}$, then its leakage at steady-state is }
\begin{eqnarray} \label{eq:leakage}
\mathcal{L}(\mydb) =  \mathop{\mathrm{max}}\limits_{\substack{\mydb' \in \mathrm{Adj}(\mydb) \\ \mathbf{x} \in \mathcal{X}_{\mycrn}(\mydb) \cup \mathcal{X}_{\mycrn}(\mydb')}} \left | \sum_{i=1}^{\mynumstate} x_i \mathrm{ln} \frac{\myxss{i}}{\myxss{i}'} - \myxss{i} + \myxss{i}'  \right |
\end{eqnarray}
{where \myxss{} and $\myxss{}'$ are the population steady-states from Eq.~\eqref{eq:dynamics} resulting from \mydb and its adjacent database \mydb'.}
\end{proposition}
\begin{proof}
Starting with Eq.~\eqref{eq:diff_priv}, and using query  $q_{\mycrn}(\mathbf{x}) = \mathbf{x}$, we have
	\begin{equation}
		\mathcal{L}(\mydb) = \mathop{\mathrm{max}}\limits_{\mydb' \in \mathrm{Adj}(\mydb)} \left | \mathrm{ln} \frac{\mathbb{P}[f_{\mycrn}(\mathbf{x_0}(\mydb),\tau)]}{\mathbb{P}[f_{\mycrn}(\mathbf{x_0}(\mydb'),\tau)]} \right |.
	\end{equation}
At steady-state we have $\lim_{\tau \rightarrow \infty} \mathbb{P}[f_{\mycrn}(\mathbf{x_0}(\mydb),\tau)] = \bar{\pi}_{\mydb}(\mathbf{x})$, hence
	\begin{equation}
		\mathcal{L}(\mydb) = \mathop{\mathrm{max}}\limits_{\substack{\mydb' \in \mathrm{Adj}(\mydb) \\ \mathbf{x} \in \mathcal{X}_{\mycrn}(\mydb) \cup \mathcal{X}_{\mycrn}(\mydb')}} \left | \mathrm{ln} \frac{\bar{\pi}_{\mydb}(\mathbf{x})}{\bar{\pi}_{\mydb'}(\mathbf{x})} \right |.
	\end{equation}
Continuing with Eq.~\eqref{eq:crn_stationary} we get
	\begin{equation}
		\mathcal{L}(\mydb) = \!\!\!\!\!\!\!\!\!\!\!\!\!\! \mathop{\mathrm{max}}\limits_{\substack{\mydb' \in \mathrm{Adj}(\mydb) \\ \mathbf{x} \in \mathcal{X}_{\mycrn}(\mydb) \cup \mathcal{X}_{\mycrn}(\mydb')}} \left | \mathrm{ln}\left( \prod_{i=1}^{\mynumstate} \frac{\myxss{i}^{x_i}}{x_i !} e^{-\myxss{i}}\right) -  \mathrm{ln}\left( \prod_{i=1}^{\mynumstate} \frac{\myxss{i}'^{x_i}}{x_i !} e^{-\myxss{i}'}\right) \right |,
	\end{equation}
	which yields Eq.~\eqref{eq:leakage}.
	\qed
\end{proof}
\begin{corollary}\label{th:leakage_y}
	{If a networked robotic system defined by database \mydb and CRN \mycrn is complex-balanced and irreducible, and is observed through query $q_{\mycrn}(\mathbf{x}) = \mathbf{y}$, with $\mathbf{y} \in \mathbb{N}_{\geq 0}^{\mynumobstate}$ and each $y_i$ of the form $\sum_{j \in \Omega_i} x_j$, with $\Omega_i \subset \{1,\ldots,\mynumstate\}$, 
	and all $\Omega_i$ disjoint is}
	\begin{equation}
		\mathcal{L}(\mydb) = \mathop{\mathrm{max}}\limits_{\substack{\mydb' \in \mathrm{Adj}(\mydb) \\ \mathbf{y}}} \left | \mathrm{ln} \frac{\mathop{\sum}\limits_{{\{ \mathbf{x} | \mathbf{y} = q_{\mathcal{N}}(\mathbf{x})  \wedge \mathbf{x} \in \mathcal{X}_\mathcal{N}(\mydb) \}}} \bar{\pi}_{\mydb}(\mathbf{x})}
		{\mathop{\sum}\limits_{{\{ \mathbf{x} | \mathbf{y} = q_{\mathcal{N}}(\mathbf{x}) \wedge \mathbf{x} \in \mathcal{X}_\mathcal{N}(\mydb') \}}} \bar{\pi}_{\mydb'}(\mathbf{x})} \right |.
	\end{equation}
\end{corollary}
This formulation, even though less compact than Eq.~\eqref{eq:leakage} above, still provides a fast means of computing the leakage for complex-balanced systems --- in particular, the alternative to using this formulation is to compute the probability mass function $\mathbb{P}$ via the Chemical Master Equation~\citep{Prorok:2016vw}, which, in our experience, is at least one order of magnitude slower.
Moreover, we note that Eq.~\eqref{eq:leakage} is linear in $\mathbf{x}$, and can, thus, be solved by integer linear programming (ILP) methods. In summary, the analytical formulation for the privacy of complex-balanced systems allows us to compute the leakage efficiently. 

\subsection{General Mechanisms}\label{sec:nonCB} 
In the general case of nonlinear dynamics, we may not be able to show that the underlying CRN is complex-balanced, and hence, may not be able to derive a stationary probability distribution. Also, we may explicitly need to analyze privacy during transient, non-steady-state, behavior. In this section, we detail a method that enables the analysis of privacy for heterogeneous systems with arbitrary dynamics.

The most widely-used computational method for obtaining the time-dependent behavior of the state of a CRN is the Stochastic Simulation Algorithm (SSA)~\citep{Gillespie:1977ww}. The basic idea behind this algorithm is to use the propensity rates to evaluate which reaction is most likely to happen within a given time interval. The result of the algorithm is a sample state trajectory. To obtain meaningful statistical information, the algorithm needs to be repeated a large number of times, which is computationally expensive overall.
An alternative approach consists of evaluating the CME, cf. Eq.~\ref{eq:cme}.
Since approaches based on the evaluation of the CME tend to be more precise than those based on SSA (when they are computationally tractable), we adopt a solution that builds on the former approach. The CME is given by the linear ordinary differential equation in Eq.~\eqref{eq:cme}.
As previously noted in Section~\ref{sec:dynamics}, when the number of possible system-level states $|\mathcal{X}_{\mycrn}(\mydb)|$ is small, it is possible to obtain a closed-form solution to Eq.~\eqref{eq:cme}.
However, when $|\mathcal{X}_{\mycrn}(\mydb)|$ is large or even infinite, it may become computationally intractable to solve the system. In such cases, we can resort to Finite State Projection (FSP) methods~\citep{Munsky:2006es} that approximate the solution by compressing the number of possible states (and, hence, also the size of $\mathbf{K}$).
Finally, we compute the probability of the observable state $\mathbf{y}$ for a given time $\tau$, which we can then plug into our formula for differential privacy, Eq.\eqref{eq:diff_priv}.
This is straightforward since $\pi_{\mathbf{x}}(\tau)$ is equivalent to $\mathbb{P}[\mathbf{x}(\tau)]$, and thus, $\mathbb{P}[q \circ f_{\mycrn}(\mathbf{x_0}(\mydb),\tau)]$ is equivalent to $\pi_{\mathbf{y}}(\tau)$, where $\pi_{\mathbf{y}}(\tau) = \sum_{\forall \mathbf{x}\,\mathrm{s.t.}\,\mathbf{y} = q(\mathbf{x})} \pi_{\mathbf{x}}(\tau)$.

\section{Case Study: Complex-Balanced Systems}\label{sec:aggregation}
In this first case-study, we focus on complex-balanced systems and turn our attention to a particular type of networked robotic system that can be described by a \emph{collaboration-DAG}. See an example in Fig.~\ref{fig:agg_graph}.
\begin{SCfigure}
	\psfrag{A}[cc][][0.8]{$a^{\{1\}}$}
	\psfrag{B}[cc][][0.8]{$a^{\{2\}}$}
	\psfrag{C}[cc][][0.8]{$a^{\{3\}}$}
	\psfrag{D}[cc][][0.8]{$a^{\{4\}}$}
	\psfrag{E}[cc][][0.8]{$a^{\{5\}}$}
	\psfrag{F}[cc][][0.8]{$a^{\{1,2\}}$}
	\psfrag{G}[cc][][0.8]{$a^{\{3,4,5\}}$}
	\psfrag{H}[cc][][0.8]{$a^{\{1,1,2\}}$}
	\psfrag{I}[cc][][0.8]{$a^{\{1-5\}}$}
	\psfrag{1}[cc][][0.8]{$r_{1}, r_{2}$}\psfrag{2}[cc][][0.9]{$r_{3}, r_{4}$}
	\psfrag{3}[cc][][0.8]{$r_{5}, r_{6}$}\psfrag{4}[cc][][0.9]{$r_{7}, r_{8}$}
	\includegraphics[width=0.45\columnwidth]{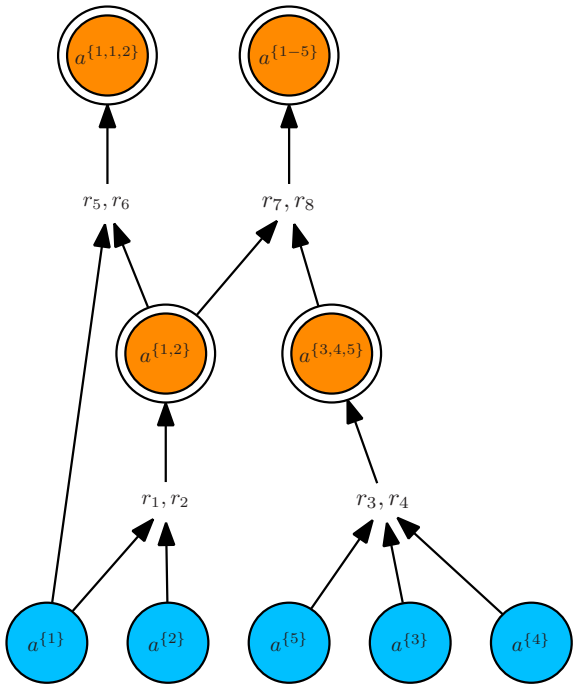}
	\caption{Example of a collaboration mechanism that is represented as a directed acyclic graph. The CRN has three elementary states and four compound states (represented by doubly outlined circles). 
	There are $L=4$ linkage classes, $\mynumcomplex = 8$ complexes, and $\mathrm{rank}(\Gamma) = 4$.
\label{fig:agg_graph}}
\end{SCfigure}

\begin{definition}[\underline{Collaboration-DAG}]\label{def:dag}
	A collaboration-DAG is a topological representation of a CRN that defines a collaboration mechanism. It is a
	directed acyclic graph, such that all \textit{state nodes} $A_i$ that compose a complex $C_j$ are connected to a \textit{reaction node} $(r_m, r_n)$ (representing both a forward and backward reaction) if and only if  there exists two reactions {\small $C_j \myarrow{\mypropx{m}}{\mypropx{n}} C_k$}, with $r_m, r_n > 0$.
\end{definition}
To apply Corollary~\ref{th:leakage_y}, we must first show that this type of network is complex-balanced.
\begin{proposition} The interaction of heterogeneous robots in a networked team with dynamics described by a CRN is a complex-balanced mechanism if the underlying CRN can be represented by a collaboration-DAG.
\end{proposition}
\begin{proof}
According to Theorem 4.1 of~\citep{Siegel:2000ui}, a CRN is complex-balanced if \emph{(i)} it is weakly reversible and \emph{(ii)} it has deficiency zero. Condition \emph{(i)} requires all complexes to be connected via some reaction pathway (cf. Def. 2.2 in~\citep{Anderson:2010ux}). If all compound states can be decomposed as well as composed, this is trivially satisfied.
The deficiency of a reaction network is $\delta = \mynumcomplex - L - \mathrm{rank}(\Gamma)$, which is the number of complexes minus the number of linkage classes, each of which is a set of complexes connected by reactions, minus the network rank, which is the rank of the stoichiometry matrix $\Gamma$. Hence, we will show that $\mynumcomplex = L + \mathrm{rank}(\Gamma)$.
From Def.~\ref{def:dag} it follows that the number of linkage classes $L$ is equal to the number of reaction nodes, and the number of complexes \mynumcomplex is equal to twice the number of reaction nodes. Thus, $\mynumcomplex = 2L$, and it remains to be shown that $\mathrm{rank}(\Gamma) = L$. Matrix $\Gamma$ is of size $\mynumstate \times \mynumreaction$, with $\mynumreaction = 2L$ (the network is weakly reversible). Since each new linkage class includes a new compound state, there are $L$ linearly independent columns in $\Gamma$, and, hence, $\mathrm{rank}(\Gamma) = L$.
\qed
\end{proof}

\subsection{Example}
\begin{SCfigure}
\psfrag{A}[cc][][0.7]{$a^{\{1\}}$}
\psfrag{B}[cc][][0.7]{$a^{\{2\}}$}
\psfrag{C}[cc][][0.7]{$a^{\{3\}}$}
\psfrag{D}[cc][][0.7]{$a^{\{1,2\}}$}
\psfrag{E}[cc][][0.7]{$a^{\{1,2,3\}}$}
\psfrag{1}[cc][][0.7]{$r_1, r_2$}
\psfrag{2}[lc][][0.7]{$r_3, r_4$}
{\includegraphics[width=0.33\columnwidth]{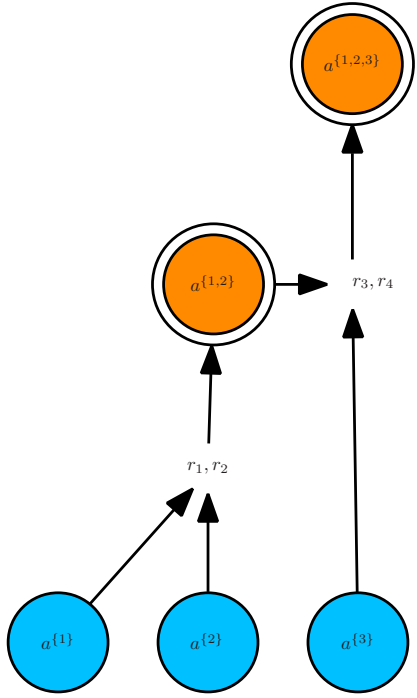}}
\caption{CRN topology corresponding to Eq.~\eqref{eq:reactions_cs1} with three elementary states and two compound states (with doubly outlines circles). 
\label{fig:graph_CB_case_study}} 
\end{SCfigure}
We consider the example shown in Fig.~\ref{fig:graph_CB_case_study}. The system is composed of three types, $\mathcal{S}= \{1,2,3\}$. Compound states are formed with one robot per type, and with types 1 and 2 interacting prior to type 3. This behavior is formalized with the following reactions:
\begin{equation} \label{eq:reactions_cs1}
	\myae{1} + \myae{2} \myarrow{\mypropx{1}}{\mypropx{2}} \myae{1,2} \mathrm{,~~~~~}
	\myae{1,2} + \myae{3} \myarrow{\mypropx{3}}{\mypropx{4}} \myae{1,2,3}
\end{equation}
The states of this system are $\mystateset = \{\myae{1}, \myae{2}, \myae{3}, \myae{1,2}, \myae{1,2,3}\}$.
Our population vector keeps track of the number of robots per state, and is written
\begin{equation}
	\mathbf{x} = [\myxe{1}, \myxe{2}, \myxe{3},\myxe{1,2}, \myxe{1,2,3}]
\end{equation}
We consider an adversary who is able to observe the number of elementary (non-collaborative) robots, the number of collaborations involving 2 robots, and the number of collaborations involving 3 robots. Hence, we formulate the observable data as $\mathbf{y} = [y_1, y_2, y_3]$ with
\begin{equation} \label{eq:observed_cs1}
	y_1 = \myxe{1} + \myxe{2} + \myxe{3},~y_2 = \myxe{1,2}, ~y_3 = \myxe{1,2,3}.
\end{equation}
Later, in Section~\ref{sec:eval_tree}, we generalize these equations with $y_i = \sum_{\{\mathcal{I} | |\mathcal{I}| = i\}} x^{\mathcal{I}}$ for all $a^{\mathcal{I}} \in \mathcal{A}$, which counts the number of compounds of a given size.

\subsubsection{Analysis}
We compute the steady-state \myxss{} by solving the deterministic system\\ \mbox{$\mymat\mathbf{A_{\kappa}}\psi(\myxss{}) = \mathbf{0}$}:
\begin{equation}\label{eq:cs1_system}
\left [
  \begin{tabular}{c c c c}
  1~ & 0~ & 0~ & 0\\
  1~ & 0~ & 0~ & 0 \\
  0~ & 1~ & 1~ & 0\\
  0~ & 1~ & 0~ & 0\\
  0~ & 0~ & 0~ & 1
  \end{tabular}
\right ] \cdot
\left [
  \begin{tabular}{@{}c@{}@{}c@{}@{}c@{}@{}c@{}}
  $-\kappa_1$ & 0 & $\kappa_2$ & 0\\
  0 & $-\kappa_3$ & 0 & $\kappa_4$ \\
  $\kappa_1$ & 0 & $-\kappa_2$ & 0\\
  0 & $\kappa_3$ & 0 & $-\kappa_4$
  \end{tabular}
\right ] \cdot
\left [
  \begin{tabular}{@{}c@{}}
  \myxess{1} \myxess{2} \\
  \myxess{1,2} \myxess{3}\\
  \myxess{1,2}\\
  \myxess{1,2,3}
  \end{tabular}
\right ] = \mathbf{0}.
\end{equation}
Since the number of robots per type is constant, we have
\begin{eqnarray}\label{eq:cs1_cons}
	\myxess{1} + \myxess{1,2} + \myxess{1,2,3} &=& \mynumrobi{1} \nonumber\\
	\myxess{2} + \myxess{1,2} + \myxess{1,2,3} &=& \mynumrobi{2} \nonumber\\
	\myxess{3} + \myxess{1,2,3} &=& \mynumrobi{3}.
\end{eqnarray}
From the equations in~\eqref{eq:cs1_cons}, we can express the variables \myxess{1}, \myxess{2}, \myxess{3} as a function of \myxess{1,2} and \myxess{1,2,3}. Using one of the five equations in~\eqref{eq:cs1_system}, we can then express \myxess{1,2,3} as a function of \myxess{1,2} through substitution. Then, proceeding with yet another equation in~\eqref{eq:cs1_system}, we write the quartic equation
{\small
\begin{eqnarray}\label{eq:quartic}
	0 = \kappa_1 &&\left(\mynumrobi{1} - \mynumrobi{3} - \myxess{1,2} + \frac{\kappa_4 \mynumrobi{3}}{\kappa_3 \myxess{1,2}+\kappa_4} \right)\cdot \nonumber\\
	&&\left( \mynumrobi{2} - \mynumrobi{3} - \myxess{1,2}  + \frac{\kappa_4 \mynumrobi{3}}{\kappa_3 \myxess{1,2}+\kappa_4} \right) - \kappa_2 \myxess{1,2},
\end{eqnarray}}%
which only depends on variable \myxess{1,2}. 
Of the four possible solutions to Eq.~\eqref{eq:quartic}, there is only a single all-positive solution (which corresponds to the single equilibrium of the complex-balanced system).
Finally, we can compute the leakage $\mathcal{L}(\mydb)$ of the observed system according to Corollary~\ref{th:leakage_y}.

\subsubsection{Evaluation}\label{sec:eval_cs1}
The observable state is a function of the system-level state, and is defined by the number of robots per type \mynumrobs, and by reaction rates $\boldsymbol{\kappa}$. Hence, we vary these values to identify their relation to privacy. We note that this relationship is made mathematically evident in Eq.~\eqref{eq:quartic}.
\newcommand{\myhh}{4.5cm}
\begin{figure}
	\centering
	\psfrag{a}[cc][][0.7]{$\mynumrobi{1}$}
	\psfrag{b}[cc][][0.7]{$\mynumrobi{2}$}
	\psfrag{c}[cc][][0.7]{$\mynumrobi{3}$}
	\psfrag{L}[cc][][0.8][90]{$\mathcal{L}$}
	\subfigure[]{\includegraphics[height=\myhh]{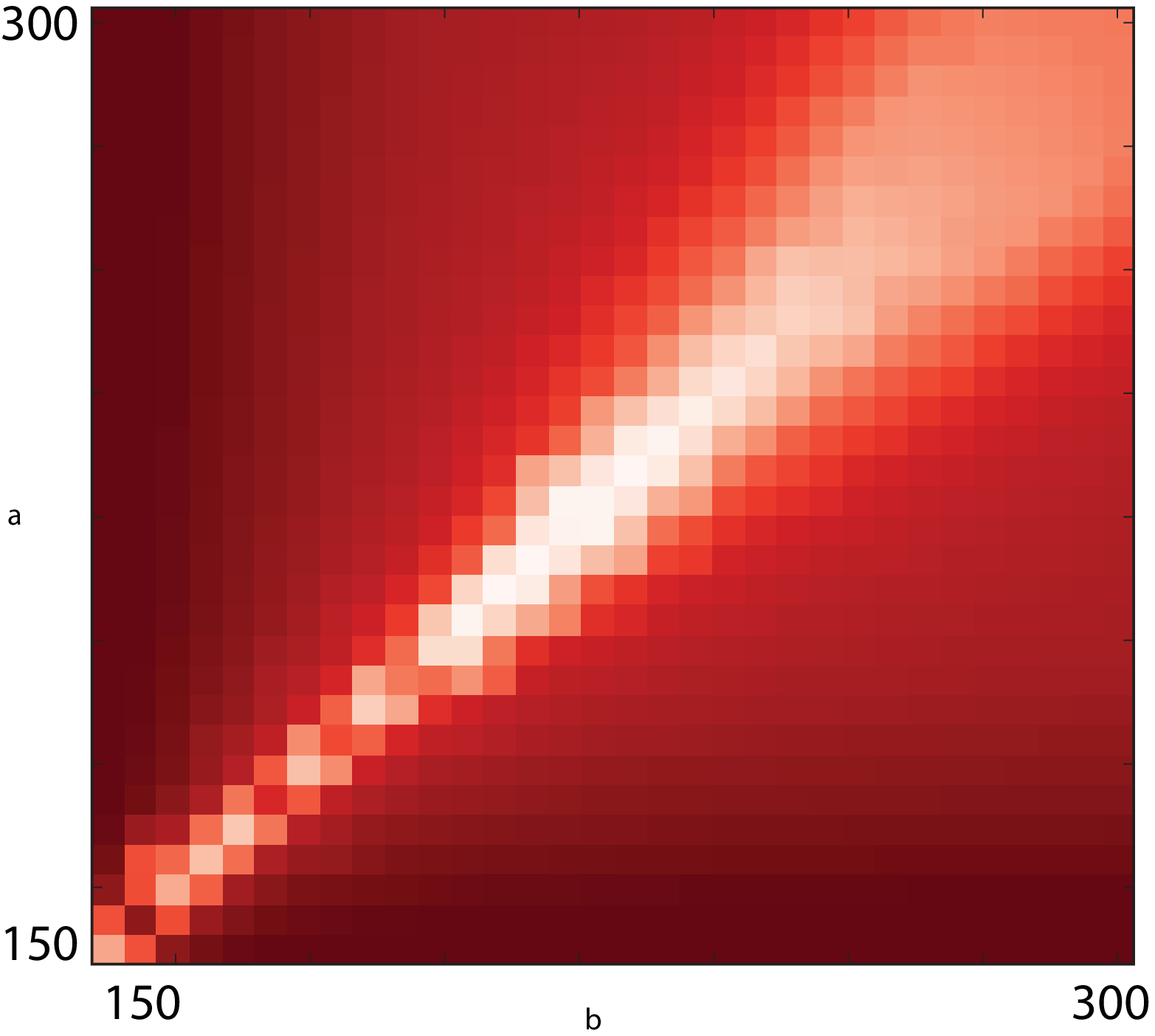}\label{fig:results_example_a}}\hspace{1cm}
	\subfigure[]{\includegraphics[height=\myhh]{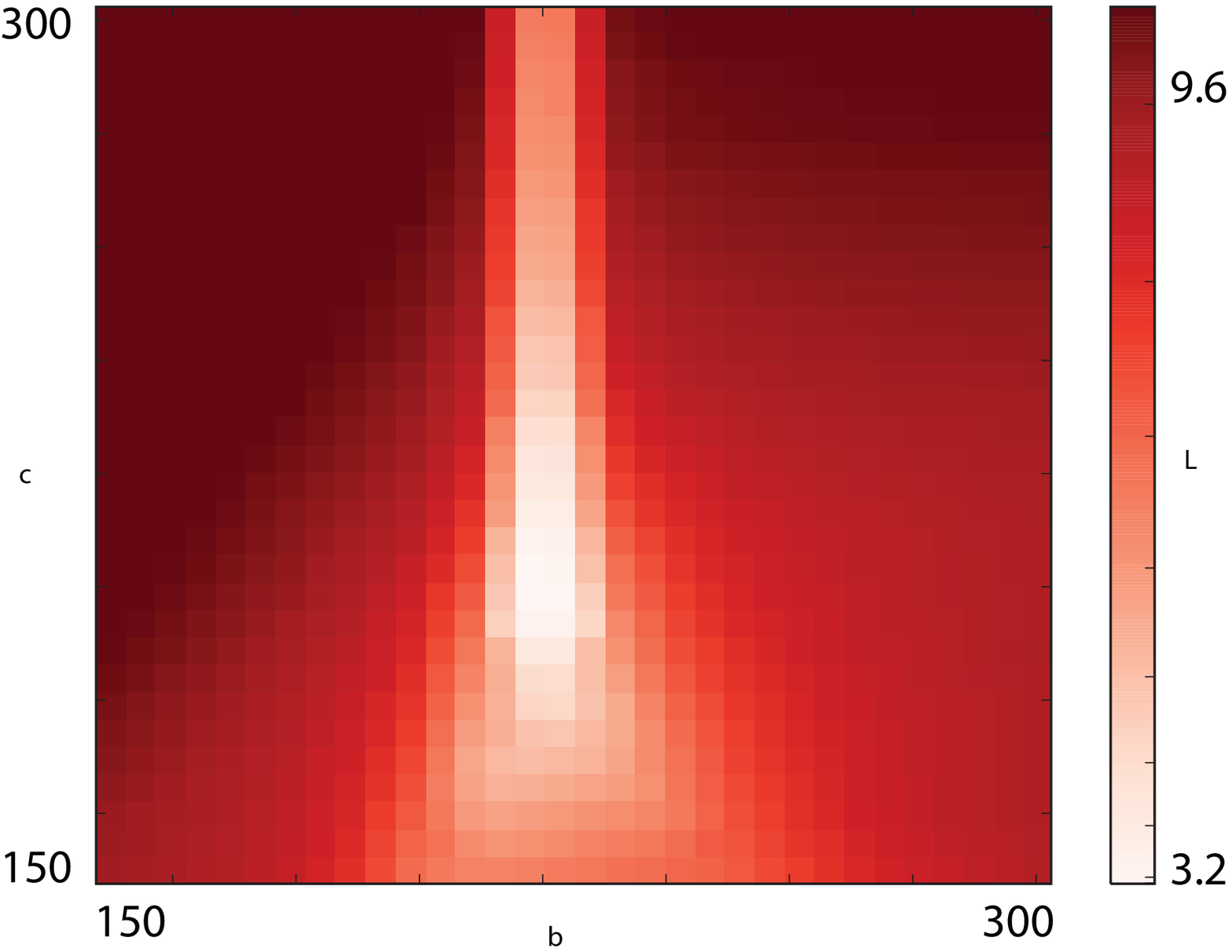}\label{fig:results_example_b}}
	\caption{Leakage for varying populations of type 1 and 2 in the range $[150,\, 300]$, while keeping the population of type 3 fixed, and with fixed reaction rates $\boldsymbol{\kappa} = 1$. In~\subref{fig:results_example_a} \mynumrobi{1} and \mynumrobi{2} vary, with \mynumrobi{3}=200, and in~\subref{fig:results_example_b} \mynumrobi{2} and \mynumrobi{3} vary, with \mynumrobi{1}=220.
\label{fig:results_example_pop}}
\end{figure}

\vspace{0.5cm}
\begin{figure}
	\centering
	\psfrag{a}[cc][][0.7]{$\kappa_2$}
	\psfrag{e}[cc][][0.7]{$\kappa_4$}
	\psfrag{o}[cc][][0.7]{$\kappa_1$}
	\psfrag{u}[cc][][0.7]{$\kappa_3$}
    \psfrag{L}[cc][][0.8][90]{$\mathcal{L}$}
	\subfigure[]{\includegraphics[height=\myhh]{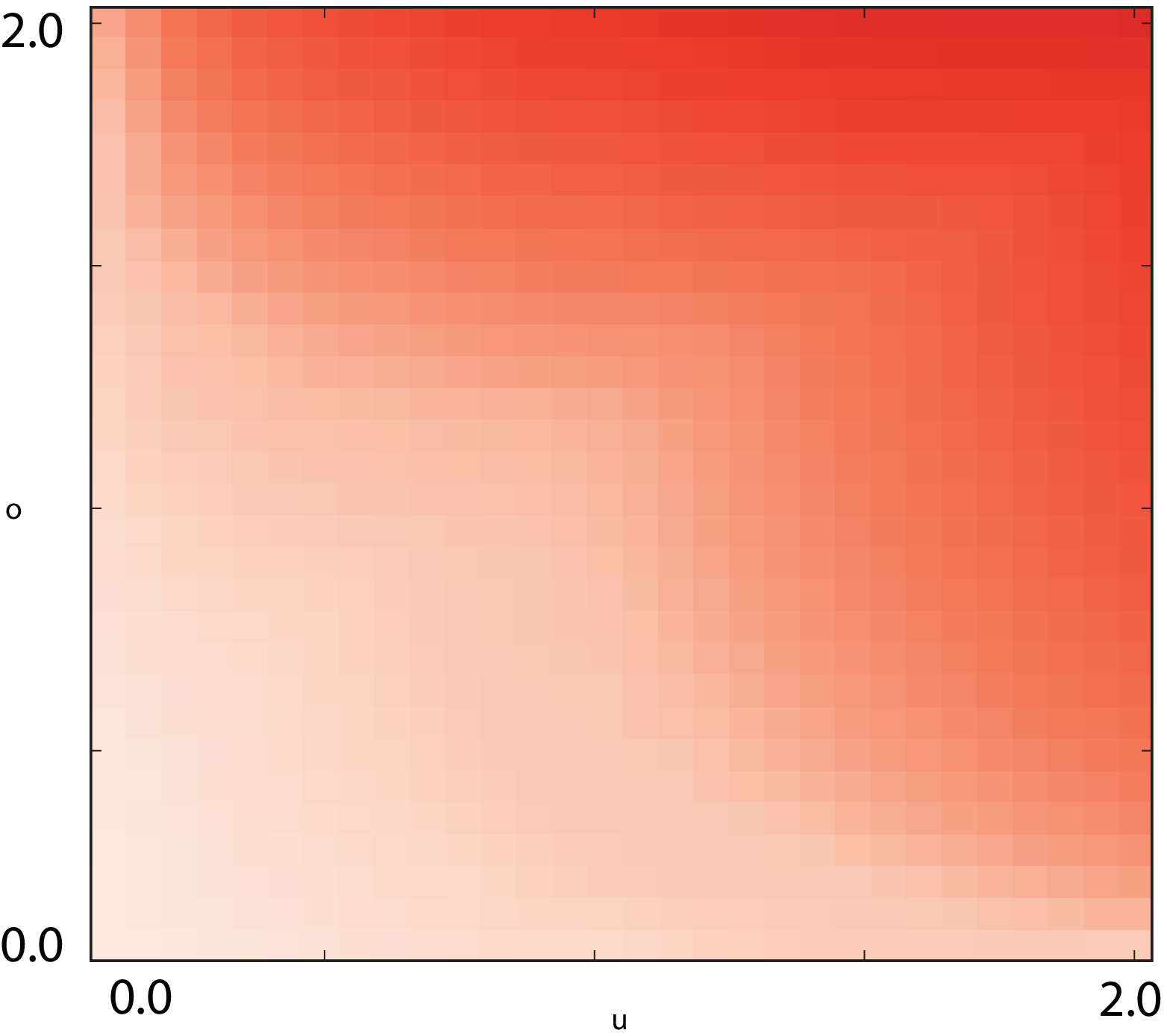}\label{fig:results_example_c}}\hspace{1cm}
	\subfigure[]{\includegraphics[height=\myhh]{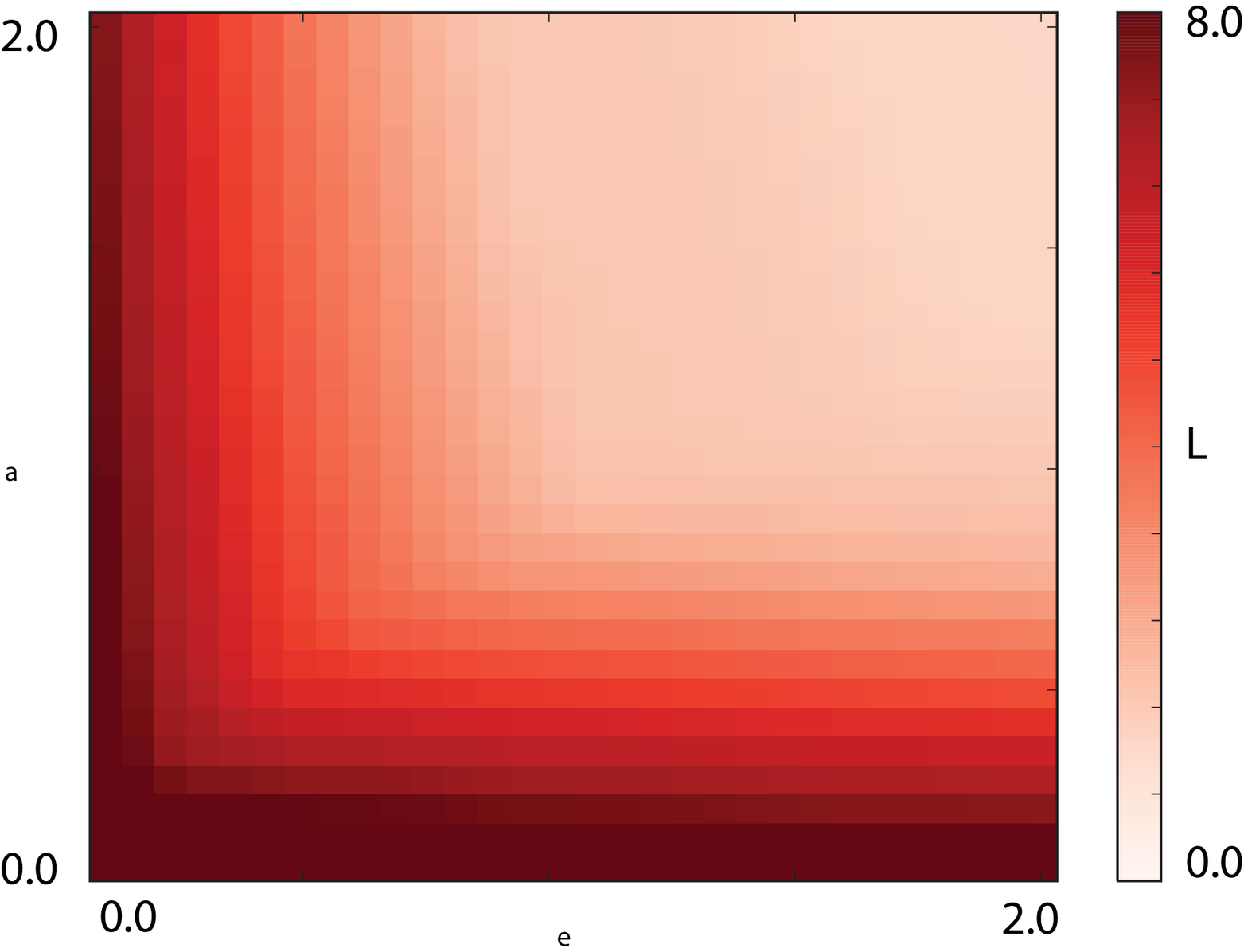}\label{fig:results_example_d}}
	\caption{Leakage for varying collaboration rates with robot populations fixed at $\mynumrobi{1} = \mynumrobi{2} = 220$, $\mynumrobi{3}=200$. In~\subref{fig:results_example_c} we vary the rates $\kappa_1,\, \kappa_3$ at which interactions are initiated, while fixing $\kappa_2 = \kappa_4 = 1$. In~\subref{fig:results_example_d} we show the reverse (varying the rates at which interactions terminate).
\label{fig:results_example_rates}}
\end{figure}

We compute the leakage for two settings, shown in Fig.~\ref{fig:results_example_pop} and Fig.~\ref{fig:results_example_rates}.
In the first setting, we fix the reaction rates and vary the robot populations of two types, while keeping the population of the third type fixed. In Fig.~\ref{fig:results_example_a}, we observe a clear ``valley" of minimum leakage values for an equal number of type 1 and 2. The overall minimum is at $\mynumrobi{1} = \mynumrobi{2} = 220,\, \mynumrobi{3}=200$. The plot also reveals that increasing the total number of robots increases privacy, as shown by the expansion of the valley in the upper right corner.
Panel Fig.~\ref{fig:results_example_b} shows the resulting leakage for varying types 2 and 3.
We observe a sharp drop in privacy as the population \mynumrobi{2} deviates from \mynumrobi{1}, with a minimum at $\mynumrobi{1} = \mynumrobi{2} = 220,\, \mynumrobi{3}=200$ (as previously observed in Fig.~\ref{fig:results_example_a}). 

We conclude that robot types 1 and 2 are interchangeable and should have a balanced number of robots for increased privacy. This symmetry is apparent, \emph{(i)} by looking at Fig.~\ref{fig:graph_CB_case_study}, where exchanging elementary states  $a^{\{1\}}$ and $a^{\{2\}}$ yields the same topology, and \emph{(ii)} by observing that equations \eqref{eq:cs1_system} and \eqref{eq:cs1_cons} remain identical when $x^{\{1\}}$ and $N^{\{1\}}$ are exchanged with $x^{\{2\}}$ and $N^{\{2\}}$, respectively.

In the second setting (Fig.~\ref{fig:results_example_rates}), we vary the reaction rates while keeping the robot populations fixed.
Figures~\ref{fig:results_example_b} and~\ref{fig:results_example_c} indicate that if we increase the probability of initiating collaborations that involve more robot types, either by increasing the number \mynumrobi{3}, or by increasing the collaboration rates, we decrease privacy. Indeed, in this setting, the collaborations of size 2 and size 3 are unique, hence, they expose more information about the system.

\subsection{Evaluating the Impact of Topology and Parameters}\label{sec:eval_tree}
To expose the impact of a CRN's topology on the leakage, we proceed by considering collaboration-DAGs that can be represented by binary trees (i.e., each reaction node has 2 in-neighbors and 1 out-neighbor). For a system composed of 16 robot types, we evaluate the leakage for each of the possible 10905 unlabeled binary rooted trees with 16 leaves (which corresponds to the Wedderburn-Etherington number). 
Fig.~\ref{fig:results_depth} shows the leakage as a function of the depth of the tree. We see a clear correlation (with a Pearson correlation coefficient 0.74 between irregular, unbalanced topologies (with greater depth) and high leakage values, and between more balanced, symmetric topologies (with smaller depth) and low leakage values.


%
\begin{SCfigure}
	\psfrag{L}[cc][][0.9][90]{$\mathcal{L}$}
	\psfrag{D}[cc][][0.9]{Depth of Tree}
	{\includegraphics[width=0.45\columnwidth]{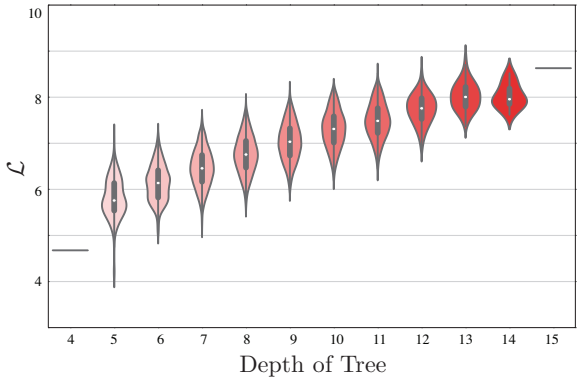}}
	\caption{Leakage for binary collaboration trees of varying topology (with 16 leaves). Trees of same depth are assembled by one violin plot that features a kernel density estimation of the underlying distribution. The Pearson correlation coefficient evaluated on this data is 0.74.
\label{fig:results_depth}}
\end{SCfigure}

\begin{SCfigure}
	\psfrag{L}[cc][][0.9][90]{$\mathcal{L}$}
	\psfrag{A}[cc][][0.9]{Average Aggregate Size}
	{\includegraphics[width=0.45\columnwidth]{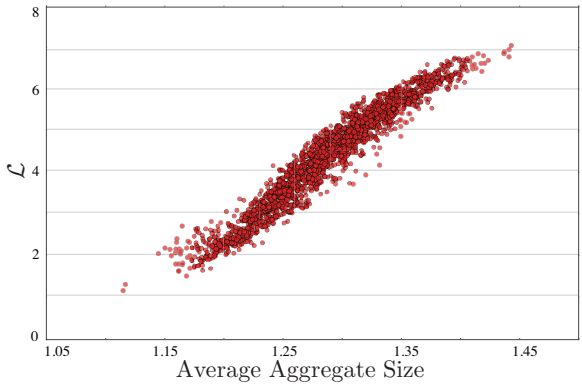}}
	\caption{Leakage for 2000 trees with identical topology (symmetric binary tree, with 8 leaves), and with all 7 collaboration rates varied uniformly and randomly in the range $[0.1, 2]$ (decomposition rates are held constant, equal to 1). The Pearson correlation coefficient evaluated on this data is 0.97. The average group size is defined by
$\sum_{\mathbf{y} \in \mathcal{Y}}  \bar{\bm{\pi}}_{\mathbf{y}} \cdot \frac{\sum_{i=1}^{N_O} y_i i}{\sum_{i=1}^{N_O} y_i}$
where $\mathcal{Y} = \{q(\mathbf{x}) | \mathbf{x} \in \mathcal{X}\}$.
\label{fig:results_agg_size}}
\end{SCfigure}
Next, to expose the impact of a CRN's parameters (i.e., collaboration rates) on the leakage, we proceed by considering a symmetric binary tree with 8 leaves (of depth 3), and we vary the collaboration rates uniformly and randomly in the range $[0.1, 2]$, gathering 2000 datapoints. The collaboration rates define the time spent in states, and ultimately, they influence the average collaboration group sizes. 

Fig~\ref{fig:results_agg_size} shows the leakage as a function of the average group size (at steady-state) for each set of rates. We see a clear correlation between the average collaboration group size and the leakage (with a Pearson correlation coefficient 0.97).

These results together indicate that privacy can be increased by \emph{(i)} designing collaboration mechanisms that are balanced (asymmetric collaboration mechanisms create more unique collaboration groups, and hence, reveal more information about the system), or by \emph{(ii)} throttling the collaboration rates (which tend to produce larger, be more unique groups).
%


\section{Case Study: General Mechanisms}
\def \labs {0.75}
\begin{SCfigure}
\psfrag{A}[cc][][0.7]{$a_{(e)}^{\{1\}}$}
\psfrag{B}[cc][][0.7]{$a_{(e)}^{\{2\}}$}
\psfrag{C}[cc][][0.7]{$a_{(e)}^{\{3\}}$}
\psfrag{D}[cc][][0.7]{$a_{(w)}^{\{1\}}$}
\psfrag{E}[cc][][0.7]{$a_{(w)}^{\{2\}}$}
\psfrag{F}[cc][][0.7]{$a_{(w)}^{\{3\}}$}
\psfrag{G}[cc][][0.7]{$a_{(w)}^{\{1,2\}}$}
\psfrag{H}[cc][][0.7]{$a_{(w)}^{\{1,3\}}$}
\psfrag{I}[cc][][0.7]{$a_{(w)}^{\{2,3\}}$}
\psfrag{a}[cc][][0.7]{$r_1$}\psfrag{b}[cc][][0.7]{$r_3$}\psfrag{c}[cc][][0.7]{$r_5$}
\psfrag{1}[cc][][0.7]{$r_{11}$}\psfrag{2}[cc][][0.7]{$r_7$}
\psfrag{3}[cc][][0.7]{$r_{15}$}\psfrag{4}[cc][][0.7]{$r_9$}
\psfrag{5}[cc][][0.7]{$r_{13}$}\psfrag{6}[cc][][0.7]{$r_{17}$}
\psfrag{7}[cc][][0.7]{$r_{21}$}\psfrag{8}[cc][][0.7]{$r_{20}$}\psfrag{9}[cc][][0.7]{$r_{19}$}
{\includegraphics[width=0.45\columnwidth]{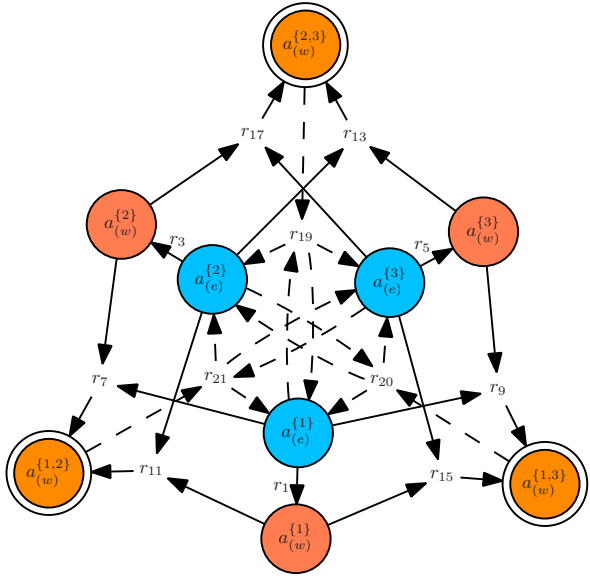}}
\caption{Topology of CRN corresponding to Eq.~\eqref{eq:reactions_cs2}. Solid black arrow are bi-directional reactions, dashed arrows arrows are uni-directional. Reverse reaction labels ($r_{2}, r_{4}, r_{6},\allowbreak r_{8}, r_{10}, r_{12}, r_{14}, r_{16}, r_{18}$) are omitted for clarity. Doubly outlined circles represent compound states. 
\label{fig:graph_nonCB}}
\end{SCfigure}
This case study considers a more general form of collaboration, which is not complex-balanced. We formulate an example of collaborative task solving: robot types have distinct (complementary) capabilities, and hence, depend on each other in order to complete tasks. Our system is composed of three types, $\mathcal{S}= \{1,2,3\}$. For any given task to be completed successfully, one robot of each type must be present at the respective task. 

There are a number of realistic scenarios that relate to this setting. A well-known work considers a setting where a homogeneous system of robots is tasked to pull sticks out of the ground~\citep{Ijspeert:2001ws} --- because the length of a single robot's arm is limited, a successful manipulation requires two robots to collaborate.
Our current case-study can be formulated analogously by expanding the original statement to a heterogeneous setting. By default, all robots are in exploration mode, searching for tasks that need to be completed. A robot encounters tasks at a certain rate. Once it has encountered a task that is either unattended, or that is occupied by one of the other two types, it will wait at the task. The robot may abandon the task (with a given rate) before it is completed, or wait until all other robots from the other types join the task. If the robot abandons the task, it returns to exploration mode. If a robot encounters a task where both other types are already present, the three robots are able to collaborate and successfully complete the task. All three robots then return to exploration mode. We formalize this behavior with the following CRN:
\begin{eqnarray}
\begin{tabular}{ll}
	\mya{e}{1} \myarrow{\mypropx{1}}{\mypropx{2}} \mya{w}{1} & 	~~~\mya{e}{2} + \mya{w}{3} \myarrow{\mypropx{13}}{\mypropx{14}} \mya{w}{2,3} \\
	\mya{e}{2} \myarrow{\mypropx{3}}{\mypropx{4}} \mya{w}{2} & 	~~~\mya{e}{3} + \mya{w}{1} \myarrow{\mypropx{15}}{\mypropx{16}} \mya{w}{1,3}  \\
	\mya{e}{3} \myarrow{\mypropx{5}}{\mypropx{6}} \mya{w}{3} & ~~~\mya{e}{3} + \mya{w}{2} \myarrow{\mypropx{17}}{\mypropx{18}}  \mya{w}{2,3} \nonumber
\end{tabular}
\end{eqnarray}
\begin{eqnarray}\label{eq:reactions_cs2}
\begin{tabular}{ll}
	\mya{e}{1} + \mya{w}{2} \myarrow{\mypropx{7}}{\mypropx{8}} \mya{w}{1,2} & ~~~\mya{e}{1} + \mya{w}{2,3} $\xrightarrow{\mypropx{19}}$ \mya{e}{1} + \mya{e}{2} + \mya{e}{3} \\
	\mya{e}{1} + \mya{w}{3} \myarrow{\mypropx{9}}{\mypropx{10}} \mya{w}{1,3} & ~~~\mya{e}{2} + \mya{w}{1,3} $\xrightarrow{\mypropx{20}}$ \mya{e}{1} + \mya{e}{2} + \mya{e}{3} \\
	\mya{e}{2} + \mya{w}{1} \myarrow{\mypropx{11}}{\mypropx{12}} \mya{w}{1,2} & ~~~\mya{e}{3} + \mya{w}{1,2} $\xrightarrow{\mypropx{21}}$ \mya{e}{1} + \mya{e}{2} + \mya{e}{3}
\end{tabular}
\end{eqnarray}
The states of this system are
\begin{eqnarray}
	\begin{tabular}{ll}
	$\mystateset^{(0)} = \emptyset$ & ~~~$\mystateset^{(2)} = \{\mya{e}{2}, \mya{w}{2}, \mya{w}{1,2}, \mya{w}{2,3}\}$ \nonumber \\
	$\mystateset^{(1)} = \{\mya{e}{1}, \mya{w}{1}, \mya{w}{1,2}, \mya{w}{1,3}\}$ & ~~~$\mystateset^{(3)} = \{\mya{e}{3}, \mya{w}{3}, \mya{w}{1,3}, \mya{w}{2,3}\}$
	\end{tabular}
\end{eqnarray}
with $\mystateset = \cup_{s=\{0,1,2,3\}} \mystateset^{(s)}$, and where $a^{\mathcal{I}}_{(e)}$ corresponds to the exploration state and $a^{\mathcal{I}}_{(w)}$ to the waiting state (e.g., $a^{\{1,3\}}_{(w)}$ corresponds to the state where one robot of types 1 and one robot of type 3 are waiting at a task for a robot of type 2).
From the reaction equations above, we see that robots interact when two robots are waiting for the remaining robot. The compound states are:
\begin{equation}
	\mystateset^{(1)} \cap \mystate^{(2)} = \{\mya{w}{1,2}\}, ~~ \mystateset^{(1)} \cap \mystate^{(3)} = \{\mya{w}{1,3}\}, ~~ \mystateset^{(2)}\cap \mystate^{(3)} = \{\mya{w}{2,3}\}
\end{equation}
Our population vector keeps track of the number of robots per state, and is:
\begin{eqnarray}
	\mathbf{x} = [\myx{e}{1}, \myx{e}{2}, \myx{e}{3}, \myx{w}{1}, \myx{w}{2}, \myx{w}{3},\myx{w}{1,2}, \myx{w}{1,3}, \myx{w}{2,3}].
\end{eqnarray}
We consider an adversary who is able to observe the number of robots that are in exploration mode, the number of robots that are waiting alone, and the number of robots that are waiting in twos. Hence, we formulate the observable data as $\mathbf{y} = [y_1, y_2, y_3]$ and
\begin{equation}
	y_1 = \myx{e}{1} + \myx{e}{2} + \myx{e}{3},~
	y_2 = \myx{w}{1} + \myx{w}{2} + \myx{w}{3},~
	y_3 = \myx{w}{1,2} + \myx{w}{2,3} + \myx{w}{1,3}
\end{equation}
The topology of this collaboration mechanism is illustrated in Fig.~\ref{fig:graph_nonCB}.
We note that the three robot types are exchangeable within this topology (yet, by definition, the robot types are heterogeneous, and hence, contribute to the collaborations through complementary skills).
The reactions' propensity rates can be attributed to the individual types. For instance, reaction $r_7$ is defined by the rate $\kappa_7$ at which type 1 encounters tasks at which type 2 is waiting. Hence, $\kappa_7$ is attributed to type 1. We define these values as $\boldsymbol{\kappa}^{(s)}$, and summarize them as $\boldsymbol{\kappa}^{(1)} = [\kappa_1, \kappa_2, \kappa_7, \kappa_8, \kappa_9,\allowbreak \kappa_{10}, \kappa_{19}]$, $\boldsymbol{\kappa}^{(2)} = [\kappa_3, \kappa_4, \kappa_{11}, \kappa_{12}, \kappa_{13}, \kappa_{14}, \kappa_{20}]$, and  $\boldsymbol{\kappa}^{(3)} = [\kappa_5,$ $\kappa_6, \kappa_{15},\allowbreak \kappa_{16}, \kappa_{17},\allowbreak \kappa_{18},\allowbreak \kappa_{21}]$.
As an example of the resulting dynamics, Fig.~\ref{fig:distributions} shows the marginal distributions resulting from $\pi_{\mathbf{x}}(\tau)$, for all nine components $x_i$ of $\mathbf{x}$, and $\pi_{\mathbf{y}}(\tau)$, for all three components $y_i$ of $\mathbf{y}$.

\subsection{Evaluation} \label{sec:eval_cs2}
\def \fighh {4.2cm}
\begin{figure}[tb]
\psfrag{1}[lc][][0.5]{\mya{e}{1}}
\psfrag{2}[lc][][0.5]{\mya{w}{1}}
\psfrag{3}[lc][][0.5]{\mya{e}{2}}
\psfrag{4}[lc][][0.5]{\mya{w}{2}}
\psfrag{5}[lc][][0.5]{\mya{e}{3}}
\psfrag{6}[lc][][0.5]{\mya{w}{3}}
\psfrag{a}[lc][][0.5]{\mya{w}{1,2}}
\psfrag{b}[lc][][0.5]{\mya{w}{1,3}}
\psfrag{c}[lc][][0.5]{\mya{w}{2,3}}
\psfrag{7}[lc][][0.5]{$y_1$}
\psfrag{8}[lc][][0.5]{$y_2$}
\psfrag{9}[lc][][0.5]{$y_3$}
\psfrag{p}[cc][][0.65][90]{Probability}
\psfrag{n}[cc][][0.65]{Number of robots}
\centering
\subfigure[]{\includegraphics[height=\fighh]{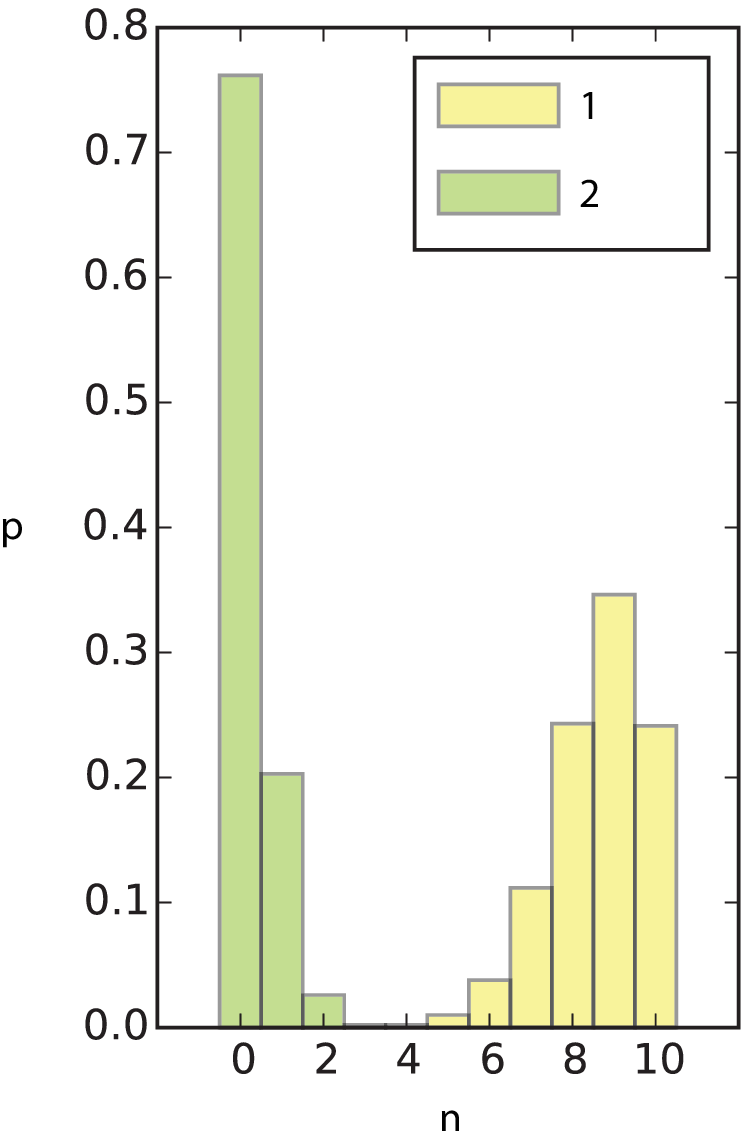}\label{fig:distr_a}}\hspace{0.2cm}
\subfigure[]{\includegraphics[height=\fighh]{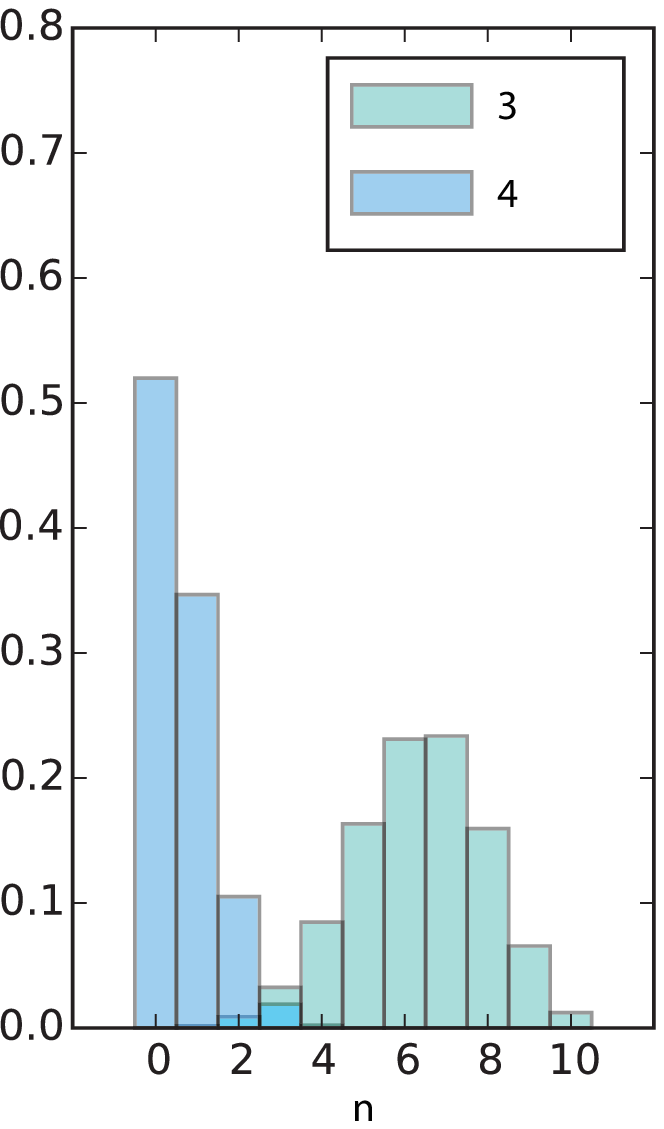}\label{fig:distr_b}}\hspace{0.2cm}
\subfigure[]{\includegraphics[height=\fighh]{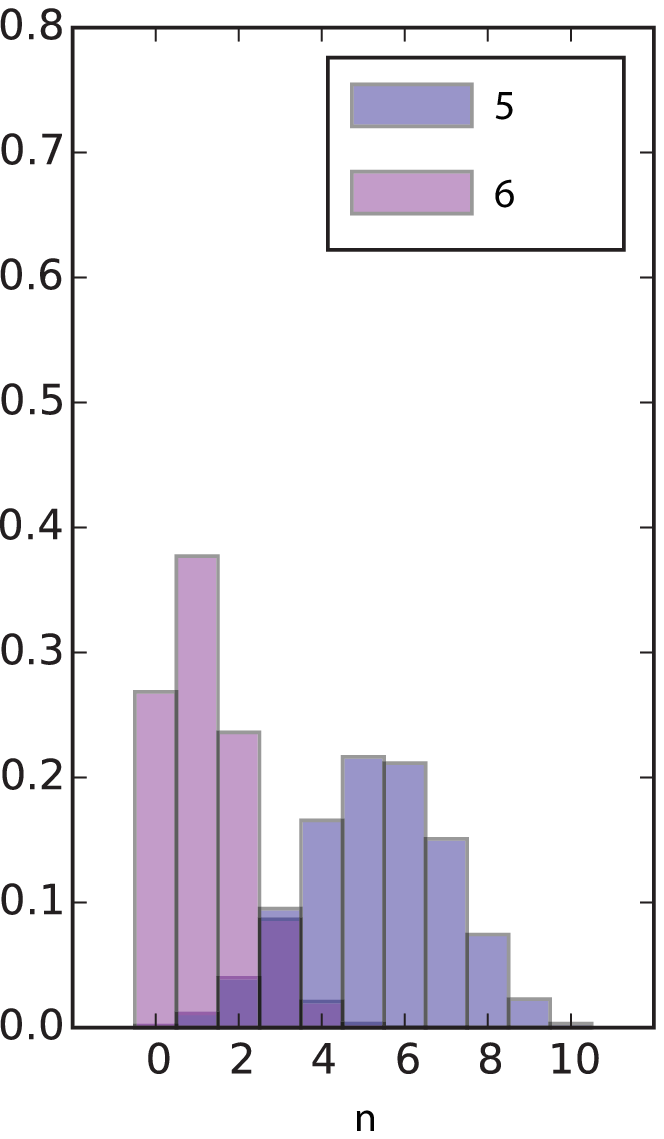}\label{fig:distr_c}}\hspace{0.2cm}
\subfigure[]{\includegraphics[height=\fighh]{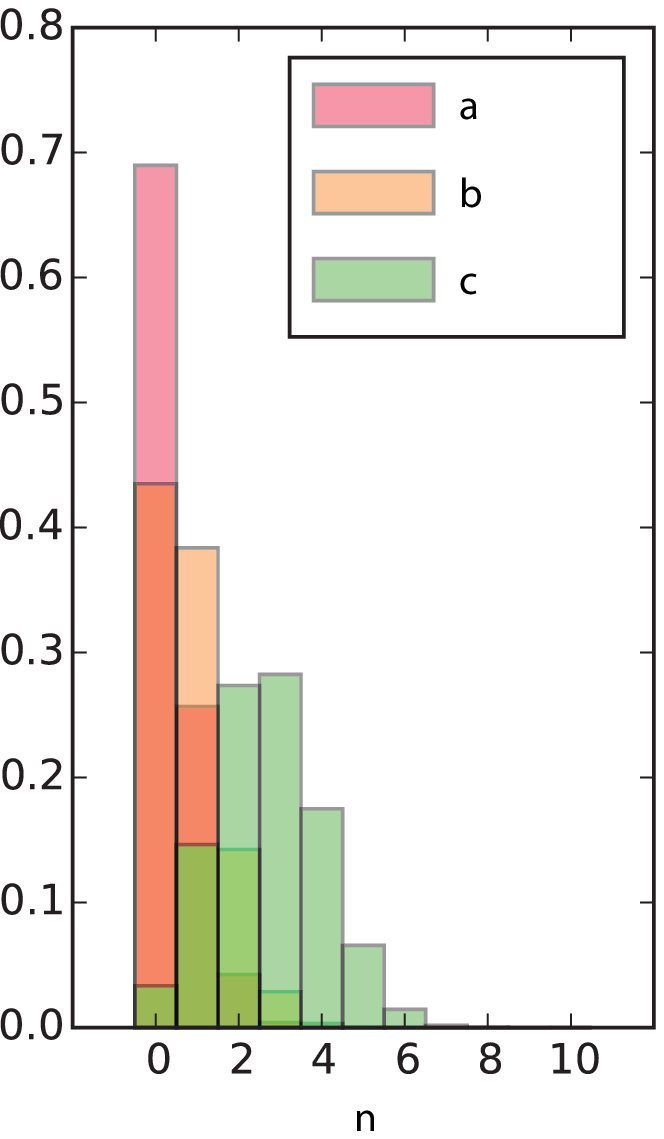}\label{fig:distr_d}}\hspace{0.2cm}
\subfigure[]{\includegraphics[height=\fighh]{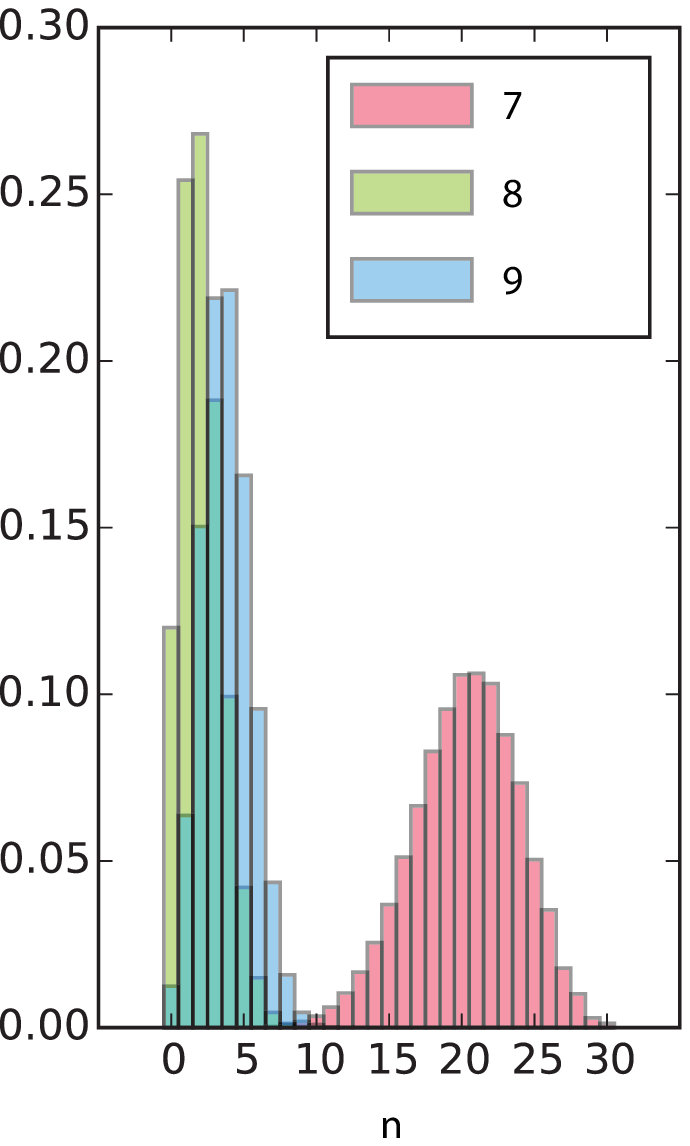}\label{fig:distr_obs}}
\caption{Panels~\subref{fig:distr_a}-\subref{fig:distr_d} show the marginal probability distributions resulting from $\pi_{\mathbf{x}}(\tau)$, for $\tau=10$, for all nine components $x_i$ of $\mathbf{x}$.
Similarly, panel~\subref{fig:distr_obs} shows the marginal distribution of $\pi_{\mathbf{y}}(\tau)$ over the components $y_i$ of the observation data $\mathbf{y}$. The data is obtained for encountering rates $\kappa_1, \kappa_{9}, \kappa_{19} = 0.2$, $\kappa_3, \kappa_{13}, \kappa_{20} = 0.5$ and $\kappa_5, \kappa_{15}, \kappa_{21} = 0.8$, and abandoning rates $\kappa_2, \kappa_{8}, \kappa_{10} = 0.8$, $\kappa_4, \kappa_{12}, \kappa_{14} = 0.5$, and $\kappa_6, \kappa_{16}, \kappa_{18} = 0.2$, for species 1, 2, 3, respectively.
\label{fig:distributions}}
\end{figure}
Since our query is a function of the system-level state, it is defined by the number of robots per type \mynumrobs, and by propensity rates $\boldsymbol{\kappa}$. By varying these values, we can show how system composition and behavior affect privacy.
We evaluate the leakage of the system for the three different settings.
First, we fix the reaction rates $\boldsymbol{\kappa} = 1$, and we vary the population \mynumrobi{2} and \mynumrobi{3} in the range $[1,\ldots,20]$ with $\mynumrobi{3}=10$.
Fig.~\ref{fig:cme_a} shows reduced leakage along the diagonal $\mynumrobi{1} = \mynumrobi{2}$. The minimum leakage occurs at $\mynumrobi{1} = \mynumrobi{2} = \mynumrobi{3} = 10$. As seen previously, in Section~\ref{sec:eval_cs1}, this result indicates that exchangeable robot types should have a similar number of robots in order to maximize privacy. In other words, since types 1, 2, and 3 are exchangeable, larger differences in the number of robots per type will produce more easily identifiable changes in the observable system-level state distributions. The plot also reveals that increasing the total number of robots increases privacy, as shown by the low leakage values in the upper right corner. Evidently, a system composed of many robots is more opaque (to an external observer): probability distributions spread over larger population ranges, and, thus, small differences in the initial population creates smaller differences in observable state distributions.

In the second and third settings, we fix the population $\mynumrobi{1} = \mynumrobi{2} = \mynumrobi{3}= 10$ and vary the reaction rates in the range $[0.2,\,2]$. Fig. \ref{fig:cme_b} shows the leakage when we vary the rates at which species 2 ($\kappa_3, \kappa_{13}, \kappa_{20}$) and species 3 ($\kappa_5, \kappa_{15}, \kappa_{21}$) encounter tasks.
Fig.~\ref{fig:cme_c} shows the leakage when we vary the rates at which species 2 ($\kappa_4, \kappa_{12}, \kappa_{14}$) and species 3 ($\kappa_6, \kappa_{16}, \kappa_{18}$) abandon tasks.
If we program the robot types with the same rates, we obtain indiscernible behaviors (since the types are exchangeable), and hence, increase privacy. This is exemplified in the plots, where off-diagonal values exhibit higher leakage, and the minimum leakage value is situated at the cell corresponding to rate uniformity. Finally, we also note that for the considered parameter ranges, varying the number of robots per type has a much larger impact on privacy than varying the behavior.
\def \figh {4cm}
\begin{figure}[tb]
\centering
\psfrag{L}[cc][][0.8]{$\mathcal{L}$}
\psfrag{u}[cc][][0.8][90]{$\mynumrobi{1}$}
\psfrag{x}[cc][][0.8]{$\mynumrobi{2}$}
\psfrag{a}[cc][][0.8][90]{$\kappa_3,\kappa_{13},\kappa_{20}$}
\psfrag{e}[cc][][0.8]{$\kappa_5,\kappa_{15},\kappa_{21}$}
\psfrag{c}[cc][][0.8][90]{$\kappa_4, \kappa_{12}, \kappa_{14}$}
\psfrag{o}[cc][][0.8]{$\kappa_6, \kappa_{16}, \kappa_{18}$}
\subfigure[]{\includegraphics[height=\figh]{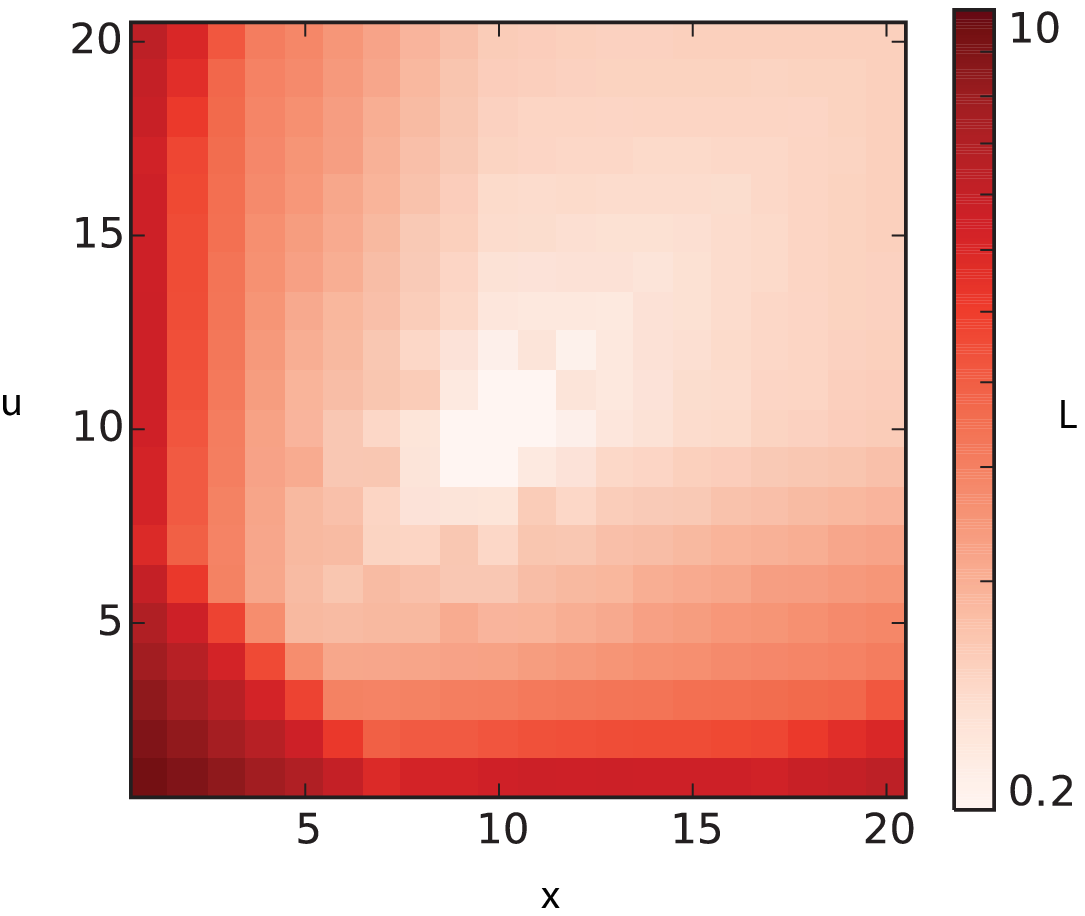}\label{fig:cme_a}}\hspace{0.2cm}
\subfigure[]{\includegraphics[height=\figh]{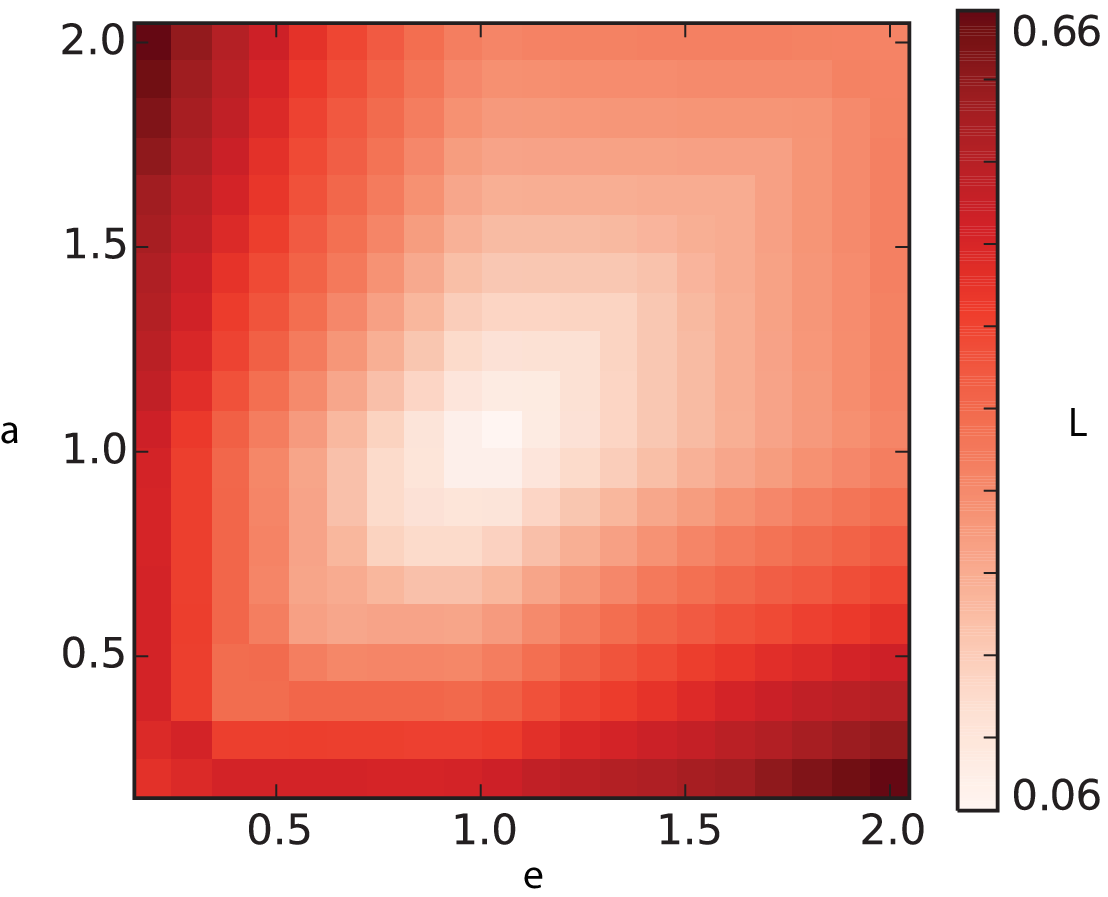}\label{fig:cme_b}}\hspace{0.2cm}
\subfigure[]{\includegraphics[height=\figh]{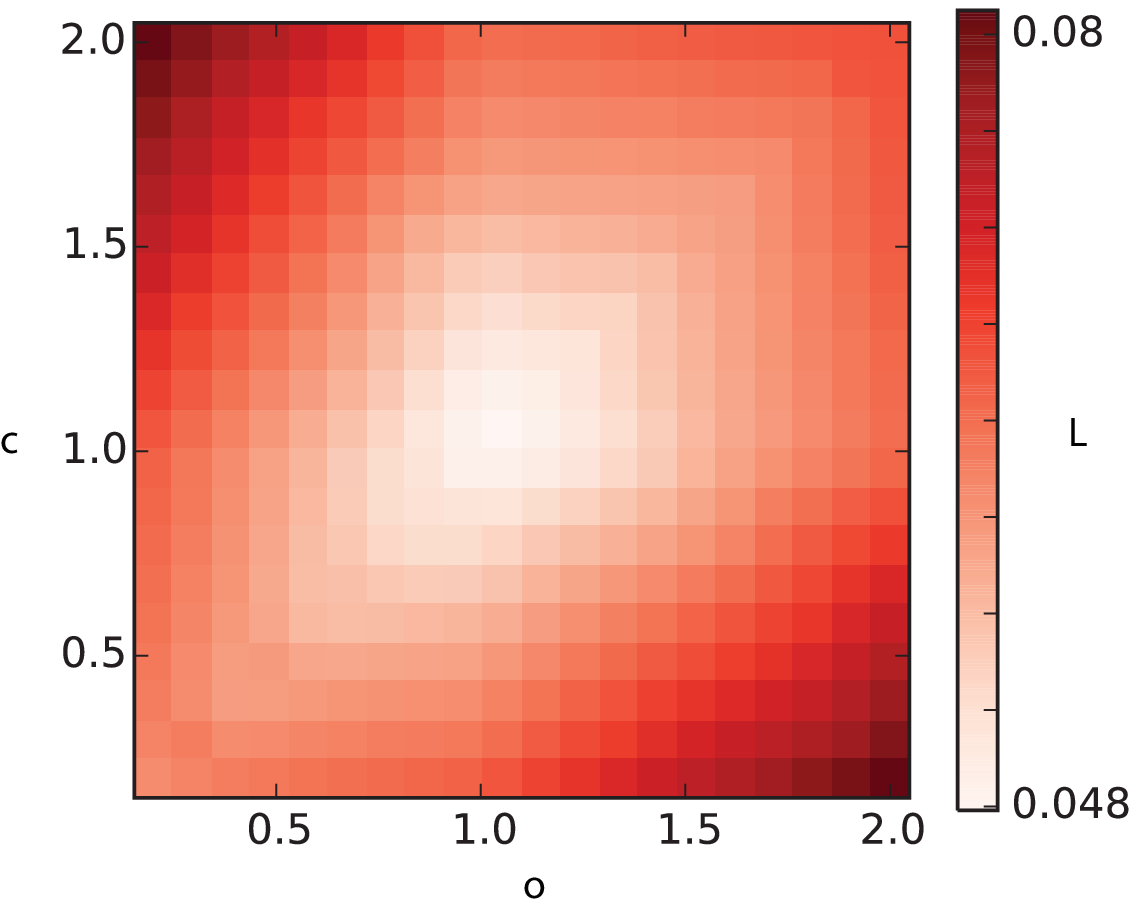}\label{fig:cme_c}}
\caption{Differential privacy of the collaboration case study. The colorbar shows the leakage. In~\subref{fig:cme_a}, we vary the population of robots in types 2 and 3 while keeping type 1 fixed. In~\subref{fig:cme_b}, we vary the task encountering rates of types 2 and 3, and in~\subref{fig:cme_c}, we vary the task abandoning rates of robot types 2 and 3, while keeping the rates of type 1 fixed.
\label{fig:results_cme}}
\end{figure}


\section{Conclusion}
\label{sec:conclusion}
In this work, we showed how to analyze the privacy of a dynamical networked robotic system composed of \emph{heterogeneous}, \emph{interdependent} robot types. Our main contribution consists of a formal definition that couples the notion of differential privacy with a model of the collaborative robot network. 
In specific, we build on the theory of Chemical Reaction Networks to formulate a macroscopic equation that describes the interactions between robot types at a system-wide level.
We showed how to evaluate the privacy through a closed-form equation, if the collaboration mechanism is complex-balanced, or numerically, if the collaboration mechanism has a general form.
We evaluated our formula on two case-studies. Our results show that we are able to determine how privacy levels vary as we vary the design parameters of the underlying system. 

Although we cast our contributions into the context of robotic teams, the methodology is applicable to a wide range of domains that study dynamical systems with interactive and interdependent agents and resources, e.g., sharing economies (collaborative consumption), and city infrastructure systems.
Privacy is an urgent and important topic --- systems that are capable of maintaining high levels of privacy are more secure and resilient. In particular, as networked robotic systems become available to all, we need to understand how to ensure their integrity so that human users are not compromised.
We intend to further this line of work by developing active privacy mechanisms that are able control the amount of information leaked, while maintaining overall system performance.

\subsubsection*{Acknowledgments. }
We gratefully acknowledge the support of ONR grants N00014-15-1-2115 and N00014-14-1-0510, ARL grant W911NF-08-2-0004, NSF grant CNS-1521617, and TerraSwarm, one of six centers of STARnet, a Semiconductor Research Corporation program sponsored by MARCO and DARPA.

\vspace{1cm}

\bibliographystyle{abbrvnat}
\bibliography{Bibliography}

\end{document}